\documentclass[ejs]{imsart}

\usepackage{graphicx} 
\usepackage[utf8]{inputenc} 
\usepackage[T1]{fontenc}    
\usepackage{url}            
\usepackage{booktabs}       
\usepackage{amsfonts}       
\usepackage{nicefrac}       
\usepackage{microtype}      

\usepackage{amsmath, amsfonts, amssymb, amsthm, dsfont}
\usepackage{algorithm}
\usepackage{algorithmic}
\usepackage{xspace}
\usepackage{natbib}
\usepackage{wrapfig}
\usepackage[hang,flushmargin]{footmisc} 

\doi{10.1214/154957804100000000}
\pubyear{0000}
\volume{0}
\firstpage{0}
\lastpage{0}

\startlocaldefs

\newlength{\minipagewidth}
\newlength{\minipagewidthx}
\newcommand{\bookboxx}[1]{\small
\par\medskip\noindent
\framebox[0.65\textwidth]{
\begin{minipage}{0.6\dimexpr\textwidth-\parindent\relax} {#1} \end{minipage} } \par\medskip }

\newcommand{\transp}{\mathsf{T}}

\newcommand{\ball}{\mathcal{B}}

\newcommand{\E}{\mathbb E}
\newcommand{\Prob}{\mathbb P}
\newcommand{\distro}{\mathcal{D}^{\ts}}

\newcommand{\rls}{\textsc{RLS}\xspace}
\newcommand{\ts}{\textsc{TS}\xspace}

\newcommand{\opt}{\text{opt}\xspace}
\newcommand{\grad}{\text{grad}\xspace}

\newcommand{\wt}[1]{\widetilde{#1}}
\newcommand{\wh}[1]{\widehat{#1}}

\newcommand{\wb}[1]{\overline{#1}}
\def\:#1{\protect \ifmmode {\mathbf{#1}} \else {\textbf{#1}} \fi}

\newcommand{\D}{\mathcal D}
\newcommand{\F}{\mathcal F}
\newcommand{\G}{\mathcal G}

\newcommand{\X}{\mathcal X}

\newcommand{\calE}{\mathcal E}

\newcommand{\I}{\mathds{1}}

\renewcommand{\Re}{\mathbb{R}}

\newtheorem{lemma}{Lemma}
\newtheorem{assumption}{Assumption}

\newtheorem{proposition}{Proposition}
\newtheorem{definition}{Definition}
\newtheorem{theorem}{Theorem}

\endlocaldefs

\begin{document}

\begin{frontmatter}

\title{Linear Thompson Sampling Revisited}
\runtitle{Linear Thompson Sampling Revisited}


\author{\fnms{Marc} \snm{Abeille}\corref{}}
\and
\author{\fnms{Alessandro} \snm{Lazaric}}
\address{INRIA Lille-Nord Europe, Team SequeL}

\runauthor{Marc Abeille and Alessandro Lazaric}

\begin{abstract}
We derive an alternative proof for the regret of Thompson sampling (\ts) in the stochastic linear bandit setting. While we obtain a regret bound of order $\wt{O}(d^{3/2}\sqrt{T})$ as in previous results, the proof sheds new light on the functioning of the \ts. We leverage the structure of the problem to show how the regret is related to the sensitivity (i.e., the gradient) of the objective function and how selecting optimal arms associated to \textit{optimistic} parameters does control it. Thus we show that \ts can be seen as a generic randomized algorithm where the sampling distribution is designed to have a fixed probability of being optimistic, at the cost of an additional $\sqrt{d}$ regret factor compared to a UCB-like approach. Furthermore, we show that our proof can be readily applied to regularized linear optimization and generalized linear model problems.
\end{abstract}



\end{frontmatter}


\section{Introduction}\label{sec:intro}


The multi-armed bandit (MAB) framework~\citep{bubeck2012regret} formalizes the exploration-exploitation trade-off in sequential decision-making, where a learner 
needs to balance between exploiting current estimates to select actions maximizing the reward and exploring actions to improve the accuracy of its estimates. Two popular approaches have been developed to trade off exploration and exploitation: the \textit{optimism in face of uncertainty} (OFU) principle (see e.g.,~\citet{agrawal1995sample,auer2002finite}), which consists in choosing the optimal action according to upper-confidence bounds on the true values, and the Thompson Sampling (\ts) strategy, which randomizes actions on the basis of their uncertainty. In this paper we mostly focus on this second approach.

TS is an heuristic for decision-making problems characterized by some unknown parameters.
The first version of this Bayesian heuristic dates back to~\citet{thompson1933likelihood}, but it has been rediscovered several times and successfully applied to address the exploration-exploitation trade-off in a wide range of problems (see e.g., \citealt{strens2000bayesian},~\citealt{chapelle2011an-empirical},~\citealt{russo2014learning}). The basic idea is to assume a \textit{prior} distribution over the unknown parameters and to use the Bayes rule to update it using the samples obtained over time. More precisely, at each time step the learner gathers information by executing the optimal action corresponding to a random parameter sampled from the current posterior distribution.\\

\textbf{Related literature.} While the Bayesian perspective of \ts provides a convenient tool to derive the sampling distribution, the algorithm is still valid in a frequentist setting, i.e., when the true parameter is not a random variable but a fixed parameter. As a result, the regret of \ts (i.e., the difference between the rewards collected by the algorithm and those of the optimal action) has been analyzed both in the Bayesian and in the frequentist setting.
In MAB, \ts has been shown to achieve optimal performance in the frequentist setting (see e.g., \citealt{may2012optimistic}, \citealt{agrawal2012thompson}, \citealt{kaufmann2012algorithmic}, \citealt{korda2013thompson}) and the dependency of the regret on its prior has been studied in the Bayesian case by~\citet{bubeck2013prior-free}. In more general cases, such as the (generalized) linear bandit and reinforcement learning settings, most of the literature focuses on the analysis of the Bayesian regret (see e.g., \cite{russo2014learning}, \cite{osband2015bootstrapped}, \cite{russo2016an-information-theoretic}). Notable exceptions are the analysis of \ts in finite MDPs by~\citet{gopalan2015thompson} and the study in linear contextual bandit (LB) by~\citet{agrawal2012thompson}.
In this paper, we focus on LB and draw novel insights on the functioning of \ts in this setting. In LB the value of an arm is obtained as the inner product between an arm feature vector $x$ and an unknown global parameter $\theta^\star$. As opposed to the OFU approach, the main technical difficulty in analyzing \ts lies in controlling the deviation in performance due to the randomness of the algorithm. \citet{agrawal2012thompson} follow the MAB proof structure (as in~\cite{agrawal2012analysis}) classifying arms as saturated and unsaturated depending on wether their standard deviation is smaller or bigger than their gap to the optimal arm.\footnote{Here we refer to the definition introduced in the \textit{arXiv} paper, which slightly differs from the original ICML paper.} While for unsaturated arms the regret is related to their standard deviation that decreases over time, they prove that \ts has a small (but constant) probability to select saturated arms and show that this guarantees a regret $\wt O \big(d^{3/2} \sqrt{T} \big)$.\\

\textbf{Contributions.} 
The major contributions of this paper are: 
\textbf{1)} Following the intuition of~\citet{agrawal2012thompson}, we show that the \ts does not need to sample from an actual Bayesian posterior distribution and that any distribution satisfying suitable concentration and anti-concentration properties guarantees a small regret. In particular, we show that the distribution should \textit{over-sample} w.r.t.\ the standard least-squares confidence ellipsoid by a factor $\sqrt{d}$ to guarantee a constant probability of being optimistic.
\textbf{2)} We provide an alternative proof of \ts achieving the same result as~\citet{agrawal2012thompson}. One of our major finding is that, leveraging the properties of support functions from convex geometry, we are able to prove that the regret is related to the gradient of the objective function, that is ultimately controlled by the norm of the optimal arms associated to any optimistic parameter $\theta$. This shows that whenever an optimistic parameter $\theta_t$ is chosen, not only is its instantaneous regret small but the corresponding optimal arm $x_t = \arg\max_x x^\transp\theta_t$ represents a \textit{useful exploration} step that improves the accuracy of the estimation of $\theta^\star$ over dimensions that are relevant to reduce regret in any subsequent non-optimistic step, which is a novel insight into the operation of \ts. This approach allows us to avoid the introduction of saturated/unsaturated arms and it illustrates why any \ts-like algorithm (not necessarily Bayesian) with a constant probability of being optimistic has a bounded regret.
\textbf{3}) Finally, we show how our proof can be easily adapted to regularized linear optimization (with arbitrary penalty) and to the generalized linear model (GLM), for which we derive the first frequentist regret bound for \ts, which was first suggested by~\citet{agrawal2012thompson} as a venue to explore.

\section{Preliminaries}\label{sec:preliminaries}

\textbf{The setting.}
We consider the stochastic linear bandit model. Let $\X \subset \Re^d$ be an arbitrary (finite or infinite) set of arms. When an arm $x\in\X$ is pulled, a reward is generated as $r(x) = x^\transp \theta^\star + \xi$, 
where $\theta^\star\in\Re^d$ is a fixed but unknown parameter and $\xi$ is a zero-mean noise. An arm $x\in\X$ is evaluated according to its expected reward $x^\transp\theta^\star$ and for any $\theta\in\Re^d$ we denote the optimal arm and its value by
\begin{equation}\label{eq:optimal_arm_definition}
x^\star(\theta) = \arg \max_{x \in \X} x^\transp \theta, \quad\quad J(\theta) = \sup_{x \in \mathcal{X}} x^\transp \theta.
\end{equation}
Then $x^\star = x^\star(\theta^\star)$ is the optimal arm for $\theta^\star$ and $J(\theta^\star)$ is its optimal value. 
At each step $t$, the learner selects an arm $x_t \in \X$ based on the past observations (and possibly additional randomization), it observes the reward $r_{t+1} = x_t^\transp\theta^\star + \xi_{t+1}$, and it suffers a \textit{regret} equal to the difference in expected reward between the optimal arm $x^\star$ and the arm $x_t$. All the information observed up to time $t$ is encoded in the filtration $\F^x_{t} = \left( \F_{1}, \sigma(x_1,r_2,\dots, r_{t},x_t) \right)$, 
where $\F_1$ contains any prior knowledge. The objective of the learner is to minimize the \textit{cumulative regret} up to step $T$, i.e., $R(T) = \sum_{t=1}^T \big( x^{\star,\transp} \theta^\star - x_t^\transp \theta^\star\big)$.\\
%

\textbf{Notations.}  We use $\|\cdot\|$ to denote the $2$-norm and $x^\transp$ to denote the transpose of $x\in\mathbb{R}^d$. For a positive definite matrix $M \in \mathbb{R}^{d\times d}$, we denote as $\| \cdot \|_{M}$ the weighted $2$-norm defined by $\|x\|_M^2 = x^\transp M x$ for any $x \in \mathbb{R}^d$. We use $\lambda_{\min}(M)$ and $\lambda_{\max}(M)$ to denote the minimum and maximum eigenvalues of the positive semi-definite matrix $M$, respectively. We use $\I\{E\}$  to denote the indicator function of the event $E$.\\

We impose the following assumptions on the problem structure and the noise $\xi_{t+1}$.

\begin{assumption}[Arm set]\label{asm:arm.set}
The arm set $\X$ is a bounded closed (and hence compact) subset of $\mathbb{R}^d$ such that $\|x\| \leq 1$ for all $x \in \X$.\footnote{A more common assumption is to assume that $\|x\| \leq X$ for all $x \in \X$. However, the structure of the proof is not affected by this modification so we set $X=1$ for sake of clarity.}
\end{assumption}

\begin{assumption}[Bandit parameter]\label{asm:param.set}
There exists $S\in\Re^+$ such that $\|\theta^\star\| \leq S$ and $S$ is known.
\end{assumption}

\begin{assumption}[Noise]\label{asm:subgaussian}
The noise process $\{\xi_t\}_t$ is a martingale difference sequence given $\F^x_t$ and it is conditionally $R$-subgaussian for some constant $R \geq 0$,
\begin{align}
\begin{split}
\forall t \geq 1,& \hspace{2mm} \mathbb{E}\left[ \xi_{t+1} | \F^x_{t} \right] = 0, \\
\forall \alpha \in \mathbb{R},& \hspace{2mm} \mathbb{E}\left[e^{\alpha \xi_{t+1}} \hspace{1mm} | \hspace{1mm} \F_{t}^x \right] \leq \exp \big( \alpha^2 R^2/2\big).
\end{split}
\label{eq:subgaussian_definition}
\end{align}
\end{assumption}

\textbf{Technical tools.} Let $(x_1,\ldots,x_t)\in\X^t$ be a sequence of arms and $(r_2,\ldots,r_{t+1})$ be the corresponding rewards, then $\theta^\star$ can be estimated by regularized least-squares (RLS). For any regularization parameter $\lambda\in\Re^+$, the design matrix and the RLS estimate are defined as
\vspace{-0.1in}
\begin{equation}\label{eq:design.matrix.rls}
V_{t} = \lambda I + \sum_{s=1}^{t-1} x_s x_s^\transp, \quad\quad\enspace \wh\theta_{t} = V_{t}^{-1} \sum_{s=1}^{t-1} x_s r_{s+1}.
\end{equation}
We recall an important concentration inequality for RLS estimates.

\begin{proposition}[Thm.~2 in~\cite{abbasi-yadkori2011improved}]\label{p:concentration}
For any $\delta \in (0,1)$, under Asm.~\ref{asm:arm.set},\ref{asm:param.set}, and~\ref{asm:subgaussian}, for any $\F^x_t$-adapted sequence $(x_1,\ldots,x_t)$, the RLS estimator $\wh\theta_t$ is such that for any fixed $t\geq 1$,
\begin{align}\label{eq:self_normalized2}
\begin{split}
&\| \wh{\theta}_t - \theta^\star \|_{V_t} \leq \beta_t(\delta), \\
\forall x \in \mathbb{R}^d, \hspace{1.5mm} &|x^\transp ( \wh{\theta}_t - \theta^\star) | \leq \|x\|_{V_t^{-1}} \beta_t(\delta),
\end{split}
\end{align}
%
w.p.\ $1-\delta$ (w.r.t.\ the noise $\{\xi_t\}_t$ and any source of randomization in the choice of the arms), where
\begin{align}\label{eq:beta}
\beta_t(\delta) = R \sqrt{2 \log \frac{(\lambda + t)^{d/2} \lambda^{-d/2}}{\delta}} + \sqrt{\lambda}S.
\end{align}
\end{proposition}

At step $t$, we define the ellipsoid $\calE^{\rls}_{t} = \big\{ \theta \in \mathbb{R}^d \hspace{1mm} | \hspace{1mm} \| \theta - \wh{\theta}_{t} \|_{V_{t}} \leq \beta_{t}(\delta^\prime) \big\}$
%
%
centered around $\wh\theta_t$ with orientation defined by $V_t$ and radius $\beta_t(\delta^\prime)$, where $\delta^\prime = \delta / 4T $.
From Eq.~\ref{eq:self_normalized2} we have that $\theta^\star\in\calE^{\rls}_t$ with high probability. 
Finally, we report a standard result of RLS that, together with Prop.~\ref{p:concentration}, shows that the prediction error on the $x_t$s used to construct the estimator $\wh\theta_t$ is cumulatively small.

\begin{proposition}\label{p:self_normalized_determinant}
Let $\lambda \geq 1$, for any arbitrary sequence $(x_1, x_2, \ldots, x_{t})\in\X^{t}$ let $V_{t+1}$ be the corresponding design matrix (Eq.~\ref{eq:design.matrix.rls}), then
%
\begin{equation}\label{eq:natural_explored_direction_ls}
\sum_{s=1}^{t} \|x_s\|_{V_s^{-1}}^2 \leq 2 \log \frac{\det(V_{t+1})}{\det(\lambda I)} \leq 2 d \log \Big( 1 + \frac{t}{\lambda} \Big).
\end{equation}
\vspace{-0.1in}
\end{proposition}

This result plays a central role in most of the proofs for linear bandit, since the regret is usually related to $||x_s||_{V_s^{-1}}$ and Prop.~\ref{p:self_normalized_determinant} is used to bound its cumulative sum. While~\citet{agrawal2012thompson} achieve this by dividing arms in saturated and unsaturated, we follow a different path that leverages the core features of the problem (structure of $J(\theta)$) and of \ts (probability of being optimistic).

\section{Linear Thompson Sampling}\label{sec:ts}

\citet{agrawal2012thompson} define \ts for linear bandit as a Bayesian algorithm where a Gaussian prior over $\theta^\star$ is updated according to the observed rewards, a random sample is drawn from the posterior, and the corresponding optimal arm is selected at each step.

\begin{figure}[t]
\begin{center}
\vspace{-0.1in}
\bookboxx{
\begin{small}
    \begin{algorithmic}[1]
        \renewcommand{\algorithmicrequire}{\textbf{Input:}}
        \renewcommand{\algorithmicensure}{\textbf{Output:}}
        \vspace{-0.02in}
        \REQUIRE $\hat{\theta}_1$, $V_1 = \lambda I$, $\delta$, $T$
        \STATE Set $\delta' = \delta/(4T)$
        \FOR{$t = \{1, \dots, T\}$}
            \STATE Sample $\eta_t \sim \distro$
            \STATE Compute parameter 
            \vspace{-0.05in}
            $$\wt\theta_{t} = \wh\theta_{t} + \beta_{t} (\delta') V_{t}^{-1/2} \eta_t$$
            \vspace{-0.15in}
            \STATE Compute optimal arm 
            \vspace{-0.05in}
            $$x_t = x^\star(\wt\theta_t) = \arg \max_{x \in \X} x^\transp \wt\theta_t$$
            \vspace{-0.14in}
            \STATE  Pull arm $x_t$ and observe reward $r_{t+1}$
            \STATE  Compute $V_{t+1}$ and $\wh\theta_{t+1}$ using Eq.~\ref{eq:design.matrix.rls}
        \ENDFOR
    \end{algorithmic}
    \vspace{-0.1in}
    \caption{\small Thompson sampling algorithm.}
    \label{alg:ts}
\end{small}
}
\vspace{-0.2in}
\end{center}
\end{figure}

As hinted by~\citet{agrawal2012thompson}, we show that \ts can be defined as a generic randomized algorithm constructed on the \rls-estimate rather than an algorithm sampling from a Bayesian posterior (see Fig.~\ref{alg:ts}). At any step $t$, given $\rls$-estimate $\wh\theta_t$ and the design matrix $V_t$, \ts samples a \textit{perturbed} parameter $\wt\theta_t$ as
\begin{equation}\label{eq:gaussian_sampling}
\wt{\theta}_{t} = \wh{\theta}_{t} + \beta_{t}(\delta^\prime) V_{t}^{-1/2} \eta_t,
\end{equation}
where $\eta_t$ is a random sample drawn i.i.d.\ from a suitable multivariate distribution $\distro$, which does not need to be associated with an actual posterior over $\theta^\star$. Then the optimal arm $x_t = x^\star(\wt\theta_t)$ is chosen, a reward $r_{t+1}$ is observed and $V_t$ and $\wh\theta_t$ are updated according to Eq.~\ref{eq:design.matrix.rls}. Notice that the resulting distribution on $\wt{\theta}_t$ is obtained by rotation of $\eta_t$ according to the design matrix $V_t$ and by a rescaling $\beta_t(\delta)$. The computational complexity of \ts is dominated by computation of $x^\star(\wt\theta_t)$, which requires solving a linear optimization problem and by the sampling process from $\distro$. This is in contrast with OFUL~\citep{abbasi-yadkori2011improved}, which requires solving a bilinear optimization problem (i.e., $\arg\max_\theta\max_x x^\transp \theta$).

The key aspect to ensure small regret is that the perturbation $\eta_t$ is distributed so that \ts explores \textit{enough} but \textit{not too much}. This translates into the following conditions on $\distro$.

\begin{definition}\label{def:ts.exploration}
$\distro$ is a multivariate distribution on $\mathbb{R}^d$ absolutely continuous with respect to the Lebesgue measure 
which satisfies the following properties:
\vspace{-0.1in}
\begin{enumerate}
\item \textit{(anti-concentration)} there exists a strictly positive probability $p$ such that for any $u \in \mathbb{R}^d$ with $\|u\| \!=\! 1$, 
$$\mathbb{P}_{\eta \sim \distro} \big(u^\transp \eta  \geq 1 \big) \geq p,$$
\vspace{-0.3in}
\item \textit{(concentration)} there exists $c,c^\prime$ positive constants such that $\forall \delta \in (0,1)$ 
$$\mathbb{P}_{\eta \sim \distro} \bigg( \|\eta\| \leq \sqrt{c d \log \frac{c^\prime d}{\delta} } \bigg) \geq 1 - \delta.$$
\vspace{-0.15in}
\end{enumerate} 
\end{definition}
%
Once interpreted in the construction of $\wt\theta_t$, the definition of $\distro$ basically requires \ts to explore far enough from $\wh\theta_t$ (anti-concentration) but not too much (concentration). This implies that \ts performs ``useful'' exploration with enough frequency (notably it performs optimistic steps), but without selecting arms with too large regret.
Let $\gamma_t(\delta) = \beta_t(\delta^\prime)\sqrt{c d \log(c^\prime d/\delta)}$, then we introduce the high-probability ellipsoid $\calE^{\ts}_{t} = \{ \theta \in \mathbb{R}^d \hspace{1mm} | \hspace{1mm} \| \theta - \wh{\theta}_{t} \|_{V_{t}} \leq \gamma_{t}(\delta^\prime)\}$.
%
%
The difference between $\calE^{\rls}_{t}$ and $\calE_t^\ts$ lies in the additional factor $\sqrt{d}$ in the definition of $\gamma_t(\delta)$ and it is crucial for both concentration and anti-concentration to hold at the same time. In Sect.~\ref{sec:proof} we prove that any distribution satisfying the conditions in Def.~\ref{def:ts.exploration} introduces the right amount of randomness to achieve the desired regret without actually satisfying any Bayesian assumption. Def.~\ref{def:ts.exploration} includes the Gaussian prior used by~\citet{agrawal2012thompson}, but also other types of distributions such as the uniform on the unit ball $\mathcal{B}_d(0,\sqrt{d})$ or distributions concentrated on the boundary of $\calE^\ts_t$ (refer to App.~\ref{sec:app_examples} for exact values of $c$, $c'$, and $p$ for uniform and Gaussian distributions).


\section{Sketch of the proof}\label{sec:geometry}

In this section we report a sketch of the proof providing a geometric intuition on the behavior of \ts and how its actions (i.e., the sampled $\wt\theta_t$ and the corresponding $x_t$) influence the regret.
For the sake of illustration, we consider the unit ball $\X = \{ \|x\| \leq 1 \} $, such that the optimal arm is just the projection of $\theta$ on the ball ($x^\star(\theta) = \theta/\|\theta\|$), and the optimal value is $J(\theta) = \theta^\transp \theta / \|\theta\| = \|\theta\|$.
We start by decomposing the regret using the definition of $J(\theta)$ as
%
%
\begin{align*}
R(T) &= \sum_{t=1}^T \Big(\big(x^{\star,\transp}\theta^\star - x_t^\transp \wt\theta_t\big) + \big(x_t^\transp \wt\theta_t  - x_t^\transp \theta^\star\big)\Big) \\
&=\underbrace{\sum_{t=1}^T \big(J(\theta^\star) - J(\wt\theta_t)\big)}_{R^{\ts}(T)} + \underbrace{\sum_{t=1}^T \big(x_t^\transp \wt\theta_t  - x_t^\transp \theta^\star\big)}_{R^{\rls}(T)},
\end{align*}
\vspace{-0.2in}

where $R^{\ts}$ depends on the randomization of \ts and $R^{\rls}$ mostly depends on the properties of \rls.\\

\textbf{Bounding $R^{\rls}(T)$.}
The decomposition $R^{\rls}(T) = \sum_{t=1}^T \big(x_t^\transp \wh{\theta}_t  - x_t^\transp \theta^\star\big) + \sum_{t=1}^T \big(x_t^\transp \wt \theta_t  - x_t^\transp \wh{\theta}_t \big)$,
%
%
shows that both \rls estimate $\wh\theta_t$ and \ts parameter $\wt\theta_t$ should concentrate appropriately. Since at each step $t$, $\wt\theta_t$ is sampled from $\distro$, the second term is kept under control by construction, while the first sum deals with the prediction error of RLS. 
As opposed to $R^\ts$, this error is not related to the exploration scheme and it is small for any sequence of arms. Intuitively, this is due to the fact that the RLS estimate is the minimizer of the regularized cumulative squared error $\wh\theta_{T+1} = \arg\min_\theta \big( \sum_{t=1}^{T} | r_{t+1} - x_t^\transp \theta |^2 + \lambda \| \theta \|^2 \big)$, so that $x_t^\transp \wh{\theta}_{T+1}$ is an accurate prediction \textit{on the arms observed so far}. The RLS minimizes the error in ``hindsight'' (i.e., after all rewards up to $T$) and therefore it also controls the \textit{online} error $| r_{t+1} - x_t^\transp\hat{\theta}_{t+1} |^2$, since by induction 
\begin{align*}
&\sum_{t=1}^{T} | r_{t+1} - x_t^\transp \wh\theta_{T+1} |^2 + \lambda \| \wh\theta_{T+1} \|^2 \geq \sum_{t=1}^{T} | r_{t+1} - x_t^\transp \wh\theta_{t+1} |^2 + \lambda \| \wh\theta_{1} \|^2.
\end{align*}
%

Having a small \textit{online} error also implies a small \textit{prediction} error $|r_{t+1} - x_t^\transp \hat{\theta}_{t}|^2$. In fact, using a recursive version of Eq.~\ref{eq:design.matrix.rls}, we have $\hat{\theta}_{t+1} = \hat{\theta}_{t} + V_{t}^{-1} x_t (1 + \|x_t\|^2_{V_t^{-1}})^{-1} ( r_{+1} - x_t^\transp \hat{\theta}_{t})$, which, together with $\|x_t\|^2_{V_{t}^{-1}} \leq 1/\lambda$, leads to $| r_{t+1} - x_{t}^\transp \hat{\theta}_{t+1} | \geq \frac{\lambda}{1 +  \lambda} |r_{t+1} - x_{t}^\transp \hat{\theta}_{t} | $.
%
Since the cumulative prediction error is small, then the associated regret $ \sum_{t=1}^T | x_t^\transp \wh{\theta}_t  - x_t^\transp \theta^\star|$ is also small. This result can be seen as an intrinsic \textit{on-policy} error guarantee of \rls.
Nonetheless, notice that while RLS minimizes the prediction error for any sequence of arms, this does not imply the consistency of the estimator. For instance, when the same arm $x$ is repeatedly played, the unknown parameter $\theta^\star$ is well-estimated in the direction of $x$ (thus making $R^\rls(T)$ small) but it is poorly estimated in any other directions. This shows the need for a careful exploration strategy to recover consistency and hence a sub-linear regret.\\

\textbf{Bounding $R^{\ts}(T)$.}
We denote by $R^\ts_t = J(\theta^\star)-J(\wt\theta_t)$ each term in $R^{\ts}(T)$. 
For optimistic algorithms this term is bounded by 0 at any step since w.h.p. $J(\wt\theta_t) \geq J(\theta^*)$ by construction. In the Bayesian regret analysis of \ts, this term is equal to $0$ by assumption that $\theta^*$ is drawn from the same prior as $\wt\theta_t$. On the other hand, in the frequentist analysis, we have to control the deviations caused by the random sampling of $\wt\theta_t$. This is achieved by showing that the arms selected by \ts provide ``useful'' information about $\theta^\star$ and contribute to keep the regret small. We follow three steps: \textbf{1)} we show that the regret is related to the sensitivity of $J$ w.r.t.\ the errors in estimating $\theta^\star$ and we bound the regret with the gradient of $J(\theta)$ at any \textit{optimistic} $\theta$; \textbf{2)} we show how the gradient in a point $\theta$ is intrinsically related to its corresponding optimal arm $x^\star(\theta)$; \textbf{3)} since we prove that \ts is frequently optimistic, then we can finally link $x^\star(\theta)$ to $x_t = x^\star(\wt\theta_t)$ and Prop.~\ref{p:self_normalized_determinant} allows us to finally bound the overall regret.\\[0.05in]
 %
\begin{figure}[h]
\begin{center}
\includegraphics[width=0.49\textwidth]{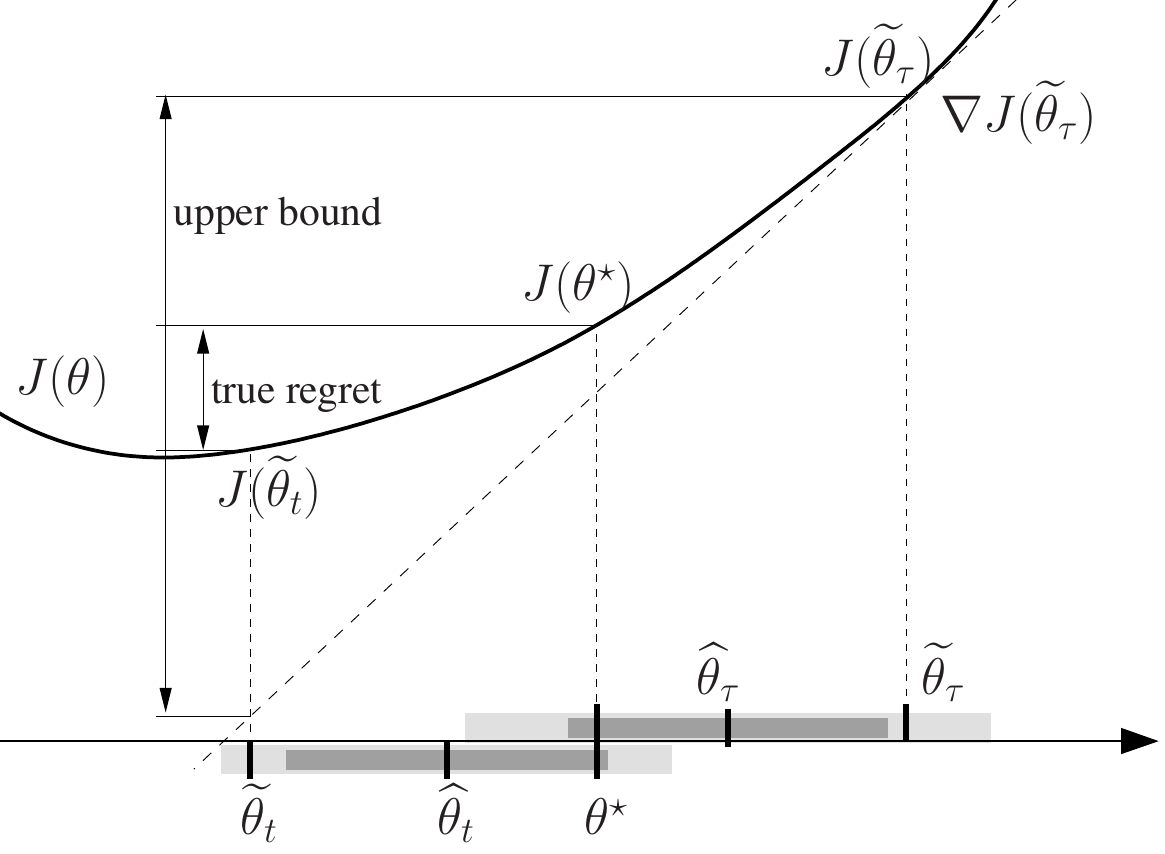}
\hspace{0.5cm}
\includegraphics[width=0.45\textwidth]{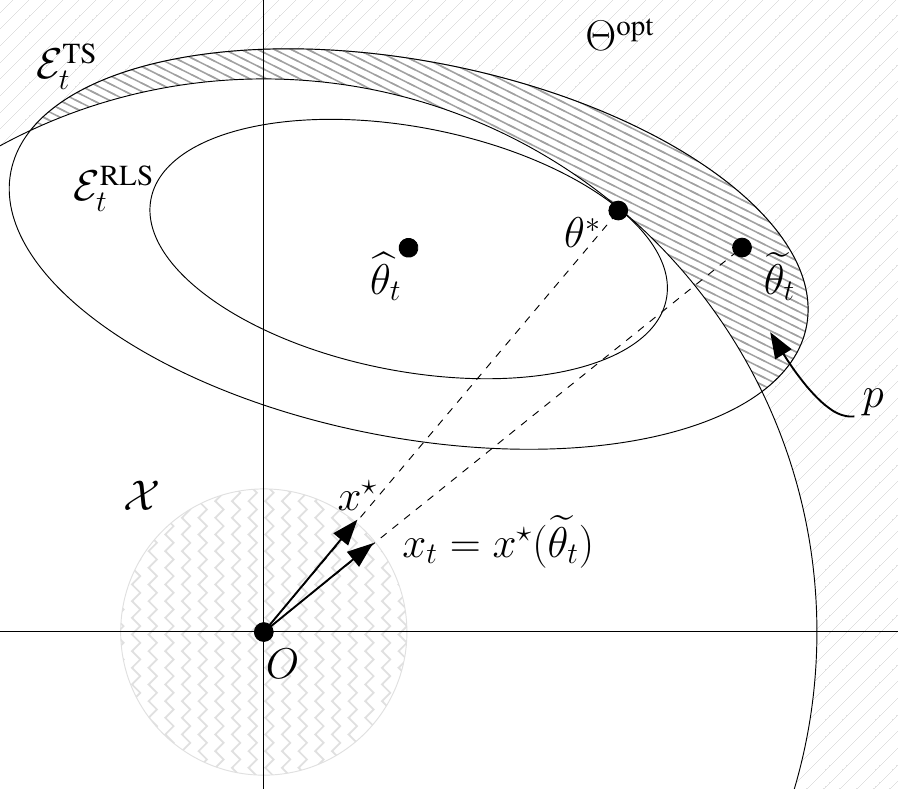}
\end{center}
\vspace{-0.1in}
\caption{\small Illustration of the steps \textbf{2)} and \textbf{3)} of the proof in $\Re^1$ and $\Re^2$. \textit{Left:} The regret at step $t$ could be bounded by the gradient of the function $J$ at a previous optimistic $\wt\theta_\tau$ times the distance between $\wt\theta_\tau$ and the current $\wt\theta_t$. Notice that $\theta^\star$ is always included in $\calE_t^\rls$ (in dark gray) and thus $\wt\theta$s sampled from $\calE_t^\ts$ (in light gray) are never too far. \textit{Right:} \ts has a constant probability of being optimistic thanks to the over-sampling of $\distro$. }
\label{fig:illustration}
\vspace{-0.1in}
\end{figure}
\textit{Step 1 (regret and sensitivity of $J$).}
We first show why the exploration of TS should be \textit{well adapted} to $J(\theta)$. Using the definition of $J(\theta) = \|\theta\|$ we have
\begin{align*}
R^\ts_t\!\!=\!J(\theta^\star\!)\! - \!J(\wt\theta_t)\! =\! \|\theta^\star\!\|\! -\! \|\wt\theta_t\| \!\leq\! \|\theta^\star\!\! -\! \wt\theta_t\| \!\leq \!\frac{\|\theta^\star\!\! -\! \wt\theta_t\|_{V_t}}{\sqrt{\lambda_{\min,t}}},
\end{align*}
where $\lambda_{\min,t}$ is the smallest eigenvalue of $V_t$. This bound shows that it is sufficient to estimate $\theta^\star$ accurately over all its components (i.e., $\lambda_{\min,t}$ tends to infinity) to obtain a no-regret algorithm. Nonetheless, the desired regret bound of $O(\sqrt{T})$ is obtained only if $\lambda_{\min,t}$ increases as $O(t)$. While this could be achieved by a fully explorative algorithm (e.g., a round robin over the canonic vectors $e_i$ reduces the ellipsoid $\calE_t^{\ts}$ to a ball of radius $\lambda_{\min,t}$), it would severely increase the second term of $R^\rls(T)$ and cause an overall linear regret\footnote{This happens because $x_t$ would be optimal w.r.t.\ a $\wt\theta_t$, which is \textit{not} in the ellipsoid $\calE_t^{\rls}$.}. Fortunately, inspecting the definition of $R^\ts_t$ reveals that not all components of $\theta^\star$ must be equally well estimated. In fact, we have w.h.p. that
\begin{align*}
R^\ts_t \leq \sup_{\theta\in\calE_t^\rls} \sup_{\theta'\in\calE_t^\ts} \big(J(\theta) - J(\theta')\big).
\end{align*}
This shows that $R^\ts_t$ is determined by the \textit{diameter} of ellipsoid $\calE_t^\ts$ w.r.t.\ $J$, which suggests that the estimation of $\theta^\star$ should be more accurate on the dimensions on which $J$ is more sensitive. In the case of $\X$ unit ball, the most sensitive direction of $J$ is $\theta^\star/\|\theta^\star\|$ itself and Fig.~\ref{fig:illustration.orientation} illustrates two opposite cases where the accuracy in the estimation of $\theta^\star$ is the same (i.e., $V_t$ has the same eigenvalues) but the regret may be very different. This is supported by the numerical experiment reported on Fig.~\ref{fig:experiment.consistency}: in the $2$ dimensional case, the eigenvalues $\lambda_{\max,t}$ and $\lambda_{\min,t}$ exhibit different divergence rate. While $1/ \lambda_{\max,t}$ decreases as $1/t$, $1/\lambda_{\min,t}$ decreases as $1 / \sqrt{t}$, which according to a direct consistency argument, is not enough to guarantee a $\sqrt{T}$ regret. However, the picture on the r.h.s shows that the diameter of the ellipsoid $\calE_t^{\ts}$ w.r.t $J$ decreases as $1/\sqrt{t}$. This is due to the fact that the sampling ellipsoid $\calE_t^{\ts}$ tends to align with the first configuration of Fig.~\ref{fig:illustration.orientation}. It implies that, on direction where $J$ is \textit{very sensitive}, the diameter of the ellipsoid shrinks appropriately (scaling with $1/ \lambda_{\max,t} \approx 1 / t$) whereas on direction where $J$ is \textit{less sensitive}, a slower decay (scaling with $1/ \lambda_{\min,t} \approx 1 / \sqrt{t}$)  of the diameter of the ellipsoid still guarantees the deviation in $J$ to be small. Therefore, the overall deviation in any direction decreases as $1/\sqrt{t}$, inducing a $\sqrt{T}$ regret.

\begin{figure}[h]
\begin{center}
\includegraphics[width=0.5\textwidth]{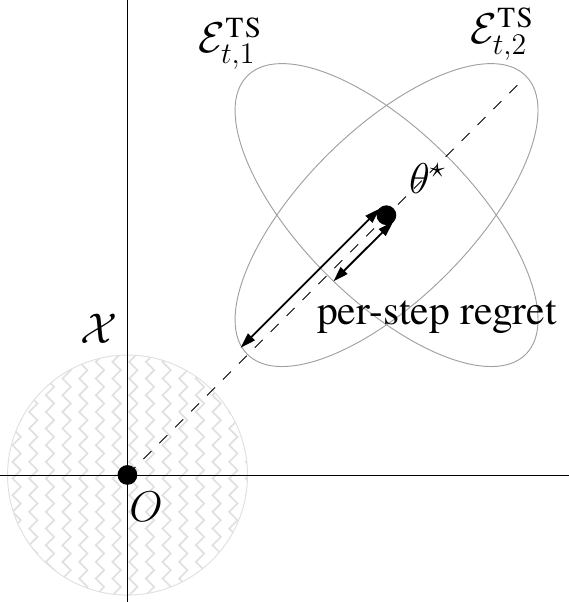}
\end{center}
\vspace{-0.1in}
\caption{\small While $\calE_{t,1}^\ts$ and $\calE_{t,2}^\ts$ have an equivalent accurate estimation of $\theta^\star$, $\calE_{t,1}^\ts$ has smaller regret than $\calE_{t,2}^\ts$.}
\label{fig:illustration.orientation}
\vspace{-0.2in}
\end{figure}
\vspace{0.1in}
\begin{figure}[!h]
\centering
\setlength{\unitlength}{\textwidth}
\begin{picture}(1,0.5)
\put(-.17,0){\includegraphics[width=1.3\textwidth]{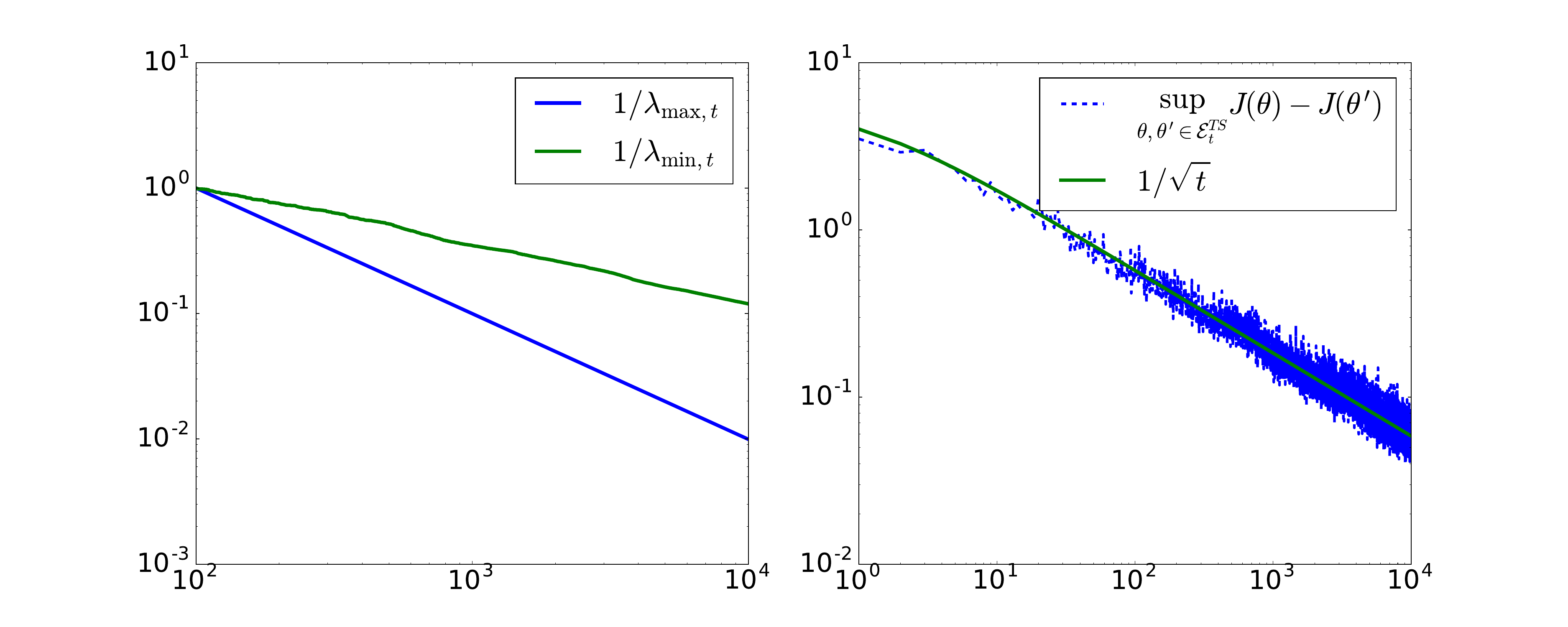}}
\end{picture}
\vspace{-0.3in}
\caption{\small Left: loglog plot of the inverse of the eigenvalues of the design matrix $V_t$ w.r.t. $t$. Rates of convergence are $1/t$ and $1/\sqrt{t}$. Right: loglog plot of the deviation in $J$ over $\calE^\ts_t$ w.r.t. $t$ (blue dashed line). The rate of convergence is $1/\sqrt{t}$ (green line).}
\label{fig:experiment.consistency}
\vspace{-0.1in}
\end{figure}

Let $\Theta^{\opt} = \{\theta: J(\theta) \geq J(\theta^\star)\}$ be the set of optimistic parameters. In our example $J(\theta) = \|\theta\|$ is convex thus we can make explicit the dependency of the regret on the sensitivity of $J$ through its gradient evaluated at any $\theta\in\Theta^{\opt}$ as (see Prop.~\ref{pr:J_support_function} for the general case)
\begin{align*}
R^\ts_t \leq \sup_{\theta'\in\calE_t^\ts} J(\theta) - J(\theta') \leq \sup_{\theta'\in\calE_t^\ts} \nabla J(\theta)^\transp (\theta-\theta'),
\end{align*}
which shows that the regret of non-optimistic $\wt\theta_t$ is bounded by the gradient of $J(\theta)$ at any optimistic $\theta$ and its distance to any other point in the \ts ellipsoid. 

\textit{Step 2 (sensitivity of $J$ and optimal arm).}
According to Prop.~\ref{p:concentration}, the difference $\theta-\theta'$ in the previous inequality is well controlled whenever $\theta$ belongs to the ellipsoid, while the first term cannot be immediately controlled by the algorithm. Nonetheless, we notice that since $J(\theta) = \|\theta\|$, then $\nabla J(\theta) = \theta/\|\theta\| = x^\star(\theta)$ (see Lem.~\ref{lem:gradient.optimal.arm} for the general case). This shows how selecting the optimal arm associated to an optimistic $\theta$ is equivalent to controlling the gradient of $J$, which results in
\begin{align*}
R^\ts_t \leq \sup_{\theta'\in\calE_t^\ts} x^\star(\theta)^\transp (\theta-\theta').
\end{align*}
From Prop.~\ref{p:self_normalized_determinant}, we could conclude that the regret would be cumulatively small if $x^\star(\theta)$ corresponded to the arms chosen by the \ts ($x_t = x^\star(\wt\theta_t)$). As a result, we need a $\theta$ \textbf{1)} that is optimistic (i.e., $\theta\in\Theta^{\opt}$), \textbf{2)} it belongs or is close to the ellipsoid $\calE_t^{\ts}$ and \textbf{3)} it is used to select an arm $x_t$. The first two requirements are at the core of the choice of the TS distribution in Def.~\ref{def:ts.exploration} where the anticoncentration property guarantees enough probability to be optimistic, while the concentration property implies that $\wt\theta$s are within a small ellipsoid.
Let $\tau<t$ be any step when \ts selects $\wt\theta_{\tau} \in \Theta^{\opt}$ with corresponding arm $x_\tau = x^\star(\wt\theta_\tau)$, then we have (see an illustration of this bound in Fig.~\ref{fig:illustration} in the 1-$d$ case)
\begin{align*}
R^\ts_t \leq \sup_{\theta'\in\calE_t^\ts} x_\tau^\transp (\wt\theta_\tau-\theta') \leq \|x_\tau\|_{V_\tau^{-1}} \sup_{\theta'\in\calE_t^\ts} \|\wt\theta_\tau-\theta'\|_{V_\tau}.
\end{align*}
Introducing $\theta^\star$ and using the fact that the design matrices forms a non-decreasing sequence (e.g. $V_\tau \leq V_t$), we decompose
\begin{align*}
 \sup_{\theta'\in\calE_t^\ts} \|\wt\theta_\tau-\theta'\|_{V_\tau} &\leq \|\wt\theta_\tau-\theta^\star\|_{V_\tau} + \sup_{\theta'\in\calE_t^\ts}  \|\theta^\star - \theta'\|_{V_\tau}\\
&\leq \|\wt\theta_\tau-\theta^\star\|_{V_\tau} + \sup_{\theta'\in\calE_t^\ts}  \|\theta^\star - \theta'\|_{V_t} 
\end{align*}
Since by Prop.~\ref{p:concentration} $\theta^\star$ is contained in all confidence ellipsoids with high probability, then 
\begin{align*}
R^\ts_t &\leq  \big( \beta_\tau(\delta^\prime) + \gamma_\tau(\delta^\prime) + \beta_t(\delta^\prime) + \gamma_t (\delta^\prime)\big) \|x_\tau\|_{V_\tau^{-1}}\\
&\leq \big( 2 \beta_T(\delta^\prime) + 2 \gamma_T(\delta^\prime) \big)\|x_\tau\|_{V_\tau^{-1}}.
\end{align*}
%
Let $K$ be the number of times $\wt\theta_t\in\Theta^{\opt}$, $t_k$ the corresponding steps, and $\nu_k = t_k - t_{k-1}$, then the final regret can be written as 
\begin{align*}
R^{\ts}(T) \leq 2 \big(\beta_T(\delta^\prime) + \gamma_T(\delta^\prime) \big) \sum_{k=1}^K \nu_k \|x_{t_k}\|_{V_{t_k}^{-1}}.
\end{align*}

\textit{Step 3 (optimism).}
This bound shows the importance that TS is optimistic with high frequency. In fact, whenever $\wt\theta_t$ is in $\Theta^{\opt}$, not only the corresponding instantaneous regret $R_t^\ts$ is upper-bounded by 0, but the exploration performed by playing arm $x^\star(\wt\theta_t)$ has also a positive impact in controlling the regret for any subsequent non-optimistic step. Consider the extreme case when TS is never optimistic, then $K=1$, $\nu_1=T$ and $R^\ts(T) = O(T)$. On the other hand, if TS is optimistic with a constant frequency, then we can easily show that $R^{\ts}(T)$ is bounded by $\wt{O}(\sqrt{T})$. Consider the case where an optimistic $\theta$ is chosen with probability $p$. Since $\E[\nu_k] = 1/p$, we can prove that w.h.p.\ $R^{\ts}(T) \leq \wt{O}(1/p\sqrt{T})$ by Cauchy-Schwarz and Prop.~\ref{p:self_normalized_determinant} applied to $\sum_{k=1}^K \|x_{t_k}\|^2_{V_{t_k}^{-1}}$, where $K\approx T$. Unfortunately, sampling $\wt\theta_t$ from the RLS ellipsoid $\calE_t^{\rls}$ may have a very small probability of being optimistic (see e.g., Fig.~\ref{fig:illustration}, where sampling uniformly in $\calE_t^{\rls}$ has zero probability to return a $\wt\theta_t \in \Theta^{\opt}$). For this reason, TS is required to draw $\wt\theta_t$ from a distribution \textit{over-sampling} by a factor $\sqrt{d}$ w.r.t.\ $\calE_t^{\rls}$ as in the definition of $\distro$. This guarantees a fixed probability $p$ of being optimistic (see Lem.~\ref{le:probability_optimistic}) and the final desired regret.


\section{Formal Proof}\label{sec:proof}

In this section we report the main steps of the regret analysis, while we postpone technical lemmas to the supplementary material. We prove the following result.

\begin{theorem}\label{th:regret_ts}
Under assumptions~\ref{asm:arm.set},\ref{asm:param.set},\ref{asm:subgaussian}, the regret of \ts is bounded w.p. $1-\delta$ as ($\delta^\prime = \frac{\delta}{4T}$)
\begin{align}\label{eq:regret_bound}
R(T) &\leq \big(\beta_T(\delta^\prime) +  \gamma_T(\delta^\prime)(1 + \frac{4}{p}) \big) \sqrt{ 2 T d \log \big( 1 + \frac{T}{\lambda} \big) } + \frac{4 \gamma_T(\delta^\prime)}{p} \sqrt{\frac{8 T } {\lambda} \log \frac{4}{\delta} }.
\end{align}

\end{theorem}

As anticipated in introduction, this bound is of order $\wt{O}(d^{3/2}\sqrt{T})$ and it entirely matches the result of~\citet{agrawal2012thompson}.
The analysis of the regret requires extra care in the definition of the filtrations. While in analyzing $R^\rls$ we consider all the knowledge up to step $t$ (i.e., including the sampled parameter $\wt\theta_t$), in $R^\ts$ we need to study the randomness of $\wt\theta_t$ conditional on all the information before sampling $\eta_t$. We introduce an additional filtration besides $\F_t^x$.

\begin{definition}
We define the filtration $\F_{t}$ as the accumulated information up to time $t$ before the sampling procedure, i.e., $\F_{t} = \left( \F_{1},\sigma(x_1,r_2,x_2,\dots,x_{t-1},r_{t-1}) \right)$.
\end{definition}

Notice that $\wh\theta_t$ and $V_{t}^{-1}$ are both $\F_{t}$ and $\F^x_{t}$ adapted, while $\wt\theta_{t}$ is a random variable w.r.t.\ $\F_{t}$ and it is fixed when considering $\F^x_{t}$. Hence we have $\F_{1} \subset \F_{2} \subset \F^x_2 \subset \F_3 \subset \F_3^x, \dots$. 
We are now ready to introduce the high-probability events we use in the rest of the proof.

\begin{definition}\label{def:concentration_events}
Let $\delta\in(0,1)$ and $\delta'=\delta/(4T)$ and $t\in[1,T]$. We define $\wh{E}_{t}$ as the event where the RLS estimate concentrates around $\theta^\star$ for all steps $s \leq t$, i.e., $\wh{E}_t = \big\{ \forall s \leq t, \hspace{2mm} \|\wh{\theta}_{s} - \theta^\star\|_{V_s} \leq \beta_s(\delta') \big\}$.
%
%
We also define $\wt{E}_{t}$ as the event where the sampled parameter $\wt{\theta}_s$ concentrates around $\wh\theta_s$ for all steps $s \leq t$, i.e., $\wt{E}_t = \big\{ \forall s \leq t, \hspace{2mm} \|\wt{\theta}_{s} - \wh{\theta}_{s}\|_{V_s} \leq \gamma_s(\delta') \big\}$.
%
%
\end{definition}

Then we have that $\wh{E}_{T} \subset \dots \subset \wh{E}_{1}$, $\wt{E}_{T} \subset \dots \subset \wt{E}_{1}$ and we use $E_t = \wh{E}_t \cap \wt{E}_t$. 

\begin{lemma}\label{le:high_proba_concentration}
Under Asm.~\ref{asm:param.set},~\ref{asm:subgaussian} we have $\mathbb{P}(\wh{E}_T \cap \wt{E}_T) \geq 1 - \frac{\delta}{2}$.
\end{lemma}

Conditioned on $\F_t$ and event $\wh{E}_t$, we have $\theta^\star\in\calE_t^\rls$, while on event $\wt{E}_t$ we have $\wt\theta_t \in \calE_t^\ts$, then we directly bound the regret as
\begin{align*}
R(T) &\leq \sum_{t=1}^T\! \big(\!J(\theta^\star) \!-\! J(\wt{\theta}_t)\!\big) \I\!\{E_t\}\! +\! \sum_{t=1}^T\! \big(\!x_t^\transp \wt{\theta}_t  \!-\! x_t^\transp \theta^\star\!\big) \I\!\{E_t\} \\
&\leq \sum_{t=1}^T R^\ts_t\I\{E_t\} \!+\! \sum_{t=1}^T R^\rls_t\I\{E_t\},
\end{align*}
w.p. $1-\delta/2$. In the interest of space we only report the formal proof to bound $R^\ts_t$, while the bound on $R^\rls(T)$ and the overall regret is postponed to App.~\ref{sec:app_proofs}.



%

Similar to the sketch in Sect.~\ref{sec:geometry}, the proof follows three steps: \textbf{1)} we use the convexity of $J$ to upper-bound the regret by its expectation conditioned on being optimistic and to relate it to the gradient of $J$, \textbf{2)} we relate the gradient of $J$ to the arms chosen by \ts over time, \textbf{3)} we show that despite the randomization, \ts has a constant probability of being optimistic.\\

\textbf{Step 1 (Regret and gradient of $J(\theta)$).}
On event $E_t$, $\wt\theta_t$ belongs to $\calE_t^{\ts}$ and thus
\begin{align*}
R_t^\ts \I\{E_t\} \leq \big(J(\theta^\star) - \inf_{\theta\in\calE_t^\ts}J(\theta)\big)\I\{\wh{E}_t\}.
\end{align*}
Recalling that $\Theta^{\opt}$ is the set of all optimistic $\theta$s, we can bound the previous expression by the expectation over any random choice of $\wt\theta$ in $\Theta_t^{\opt} := \Theta^{\opt} \cap \calE^\ts_t$ where we restrict the optimistic set to the high-probability sampling ellipsoid, that is 
\begin{align*}
R_t^\ts \leq \E\Big[\big(J(\wt\theta) - \inf_{\theta\in\calE_t^\ts}J(\theta)\big)\I\{\wh{E}_t\} \Big| \F_t, \wt\theta\in \Theta_t^{\opt}\Big],
\end{align*}
where $\wt\theta = \wh\theta_t + \beta_t(\delta^\prime)V_t^{-1/2}\eta$ with $\eta\sim\distro$ is the \ts sampling distribution. We now rely on the following characterization of $J(\theta)$ (see App.~\ref{sec:app_support}).

\begin{proposition}\label{pr:J_support_function}
For any set of arm $\X$ satisfying Asm.~\ref{asm:arm.set}, $J(\theta) = \sup_{x} x^\transp \theta$ has the following properties: \textbf{1)} $J$ is real-valued as the supremum is attained in $\X$, \textbf{2)} $J$ is convex on $\mathbb{R}^d$, \textbf{3)} $J$ is continuous with continuous first derivative except for a zero-measure set w.r.t.\ the Lebesgue's measure.
\end{proposition}

These properties follow from the fact that $J$ is the \textit{support function} of $\X$ and it shows that $J$ is convex for any arm set $\X$. As a result, we can directly relate $R_t^\ts$ to the gradient of $J$ as
\begin{align*}
&R_t^\ts \leq \E\Big[ \sup_{\theta\in\calE_t^\ts} \nabla J(\wt\theta)^\transp (\wt\theta-\theta) \I\{\wh{E}_t\} \Big| \F_t, \wt\theta\in\Theta_t^{\opt}\Big]  \\
&\leq \E\Big[  \|\nabla J(\wt\theta)\|_{V_t^{-1}} \!\!\sup_{\theta\in\calE_t^\ts}\!\!\|\wt\theta\!-\!\theta\|_{V_t}  \Big| \F_t, \wt\theta\!\in\!\Theta_t^{\opt}, \wh{E}_t \Big] \mathbb{P} ( \wh{E}_t)\\
& \leq 2 \gamma_t(\delta^\prime)  \E\Big[  \|\nabla J(\wt\theta)\|_{V_t^{-1}} \Big| \F_t, \wt\theta\!\in\!\Theta_t^{\opt}, \wh{E}_t \Big] \mathbb{P} ( \wh{E}_t)
\end{align*}
%
where we use Cauchy-Schwarz,  we ``push'' the event $\wh{E}_t$ into the conditioning and we use the fact that $\wt\theta \in \calE^\ts_t$.\\

\textbf{Step 2 (From gradient of $J(\theta)$ to optimal arm $x^\star(\theta)$).}
In the sketch of the proof there was a direct relationship between $\nabla J(\theta)$ and the optimal arm corresponding to $\theta$ by direct construction. In the next lemma, we show that this connection is true for any arm set $\X$ (proof in App.~\ref{sec:app_support}).

\begin{lemma}\label{lem:gradient.optimal.arm}
Under Asm.~\ref{asm:arm.set}, for any $\theta \in\Re^d$, we have $ \nabla J (\theta)  = x^\star(\theta)$ except for a zero-measure set w.r.t.\ the Lebesgue's measure.
\end{lemma}

This property strongly connects the exploration of \ts to the actual regret. In fact, together with Prop.~\ref{p:self_normalized_determinant}, it implies that selecting the optimal arm associated with any optimistic $\theta$ is equivalent to reducing the gradient of $J$ and ultimately the regret $R^\ts_t$. This motivates the next step where we show that since \ts is often optimistic, then the arm $x_t = x^\star(\wt\theta_t)$ contributes to the reduction of the regret.\\ 

\textbf{Step 3 (Optimism).} 
The optimism of \ts is a direct consequence of the convexity of $J$ and the fact that the distribution of $\eta$ is oversampling by a factor $\sqrt{d}$ w.r.t.\ the ellipsoid $\calE_t^{\rls}$. Since this is at the core of the TS sampling analysis, we detail the proof here but postpone convexity results in App.~\ref{sec:app_proofs}. 

\begin{lemma}\label{le:probability_optimistic}
Let $\Theta_t^{\opt}\! :=\! \{ \theta\!\in\!\Re^d |  J(\theta)\! \geq\! J(\theta^\star)\! \} \cap \calE^\ts_t$ be the set of optimistic parameters, $\wt\theta_t = \wh\theta_t + \beta_t(\delta^\prime) V_t^{-1/2} \eta$ with $\eta \sim \distro$, then $\forall t\geq 1, \hspace{1mm}\mathbb{P}\big( \wt\theta_t \in \Theta_t^{\opt} |  \F_t, \wh{E}_t \big) \geq p/2.$
%
%
\end{lemma}

\begin{proof}
We need to study the probability that a $\wt\theta$ drawn at time $t$ from the \ts sampling distribution is optimistic, i.e., $J(\wt\theta) \geq J(\theta^\star)$, under event $\wh{E}_t$. More formally let
\begin{align*}
p_t = \Prob\big( J(\wt\theta) \geq J(\theta^\star) | \F_t, \wh{E}_t \big).
\end{align*}
Using the definition of $\wh{E}_t$ we have that $\theta^\star \in \calE_t^{\rls}$ (i.e., the true parameter vector belongs to the \rls ellipsoid) and then we can replace $J(\theta^\star)$ by the supremum over the ellipsoid as
\begin{align*}
p_t \geq \Prob\Big( J(\wt\theta) \geq \sup_{\theta\in\calE_t^\rls} J(\theta) \Big| \F_t, \wh{E}_t \Big).
\end{align*}
By recalling the definition of the \ts sampling process, we can write $\wt\theta = \wh\theta_t + \beta_t(\delta^\prime) V_t^{-1/2} \eta$, where $\eta\sim\distro$ and for notational convenience, we define the function $f_t(\eta) = J(\wh\theta_t + \beta_t(\delta^\prime) V_t^{-1/2} \eta)$. Let $\wb\theta_t = \arg\max_{\theta\in\calE_t^\rls} J(\theta)$ and $\wb\eta_t$ be the corresponding $\eta$ (i.e., $\wb\eta_t$ is such that $\wb\theta_t = \wh\theta_t + \beta_t(\delta^\prime) V_t^{-1/2} \wb\eta_t$). Since the supremum is taken within $\calE_t^\rls$, $\wb\eta_t$ belongs to the unit ball (i.e., $\wb\eta_t \in\ball_{d}(0, 1)$). As a result, we can rewrite the previous expression as
\begin{align*}
p_t \geq \Prob\Big( f_t(\eta) \geq f_t(\wb\eta_t) \Big| \F_t, \wh{E}_t \Big).
\end{align*}
Since the function $f_t$ inherits all the properties of $J$, notably its convexity in $\eta$, we know that the supremum on a convex closed set is reached at least at one point $\bar{\eta}_t$  and that it belongs to the boundary (see Prop.~\ref{pr:maximum_convex_closed_set}), which in our case corresponds to $\|\wb{\eta}_t \| = 1$. Moreover, let $\mathcal{H}_t(\wb{\eta}_t)$ be the hyperplane tangent to $\wb{\eta}_t$. $\mathcal{H}_t(\bar{\eta}_t)$ splits $\mathbb{R}^{d}$ in two complementary subsets $\mathcal{G}_t$ and $\mathcal{G}_t^\perp$ where $\mathcal{G}_t$ does not contain the unit ball by convention. Formally, one has:
\begin{equation*}
\begin{aligned}
\mathcal{H}_t(\wb{\eta}_t) &:= \{ \eta \in \mathbb{R}^d \text{ s.t. } \eta^\transp \bar{\eta}_t = 1 \}, \quad
\mathcal{G}_t := \{ \eta \in \mathbb{R}^d \text{ s.t. } \eta^\transp \bar{\eta}_t \geq 1 \}.
\end{aligned}
\end{equation*}

Again, the convexity of $f_t$ ensures that $f_t(\eta) \geq  f_t(\bar{\eta}_t)$ for all $\eta \in \mathcal{G}_t$ as proved in Prop.~\ref{pr:increasing_subspace}. As illustrated in Fig.~\ref{fig:optimism.eta} the probability of being optimistic is now reduced to the probability that $\eta$ drawn from $\distro$ falls into $\G_t$, which corresponds to
\begin{align*}
p_t \geq \Prob\Big( \eta\in\G_t \Big| \F_t, \wh{E}_t \Big) = \Prob\Big( \eta^\transp \bar{\eta}_t \geq 1 \Big| \F_t, \wh{E}_t \Big).
\end{align*}
Notice that $\bar{\eta}_t$ is entirely defined by the filtration $\F_t$ and the event $\wh{E}_t$ and it is thus independent from $\wb\eta_t$. As a result, we obtain from property 1 of Def.~\ref{def:ts.exploration} of the \ts sampling distribution,  that
\begin{align*}
\Prob\Big( \eta^\transp \bar{\eta}_t \geq 1 \Big| \F_t \Big) \geq p.
\end{align*}
\begin{figure}[h]
\begin{center}
\includegraphics[width=0.8\textwidth]{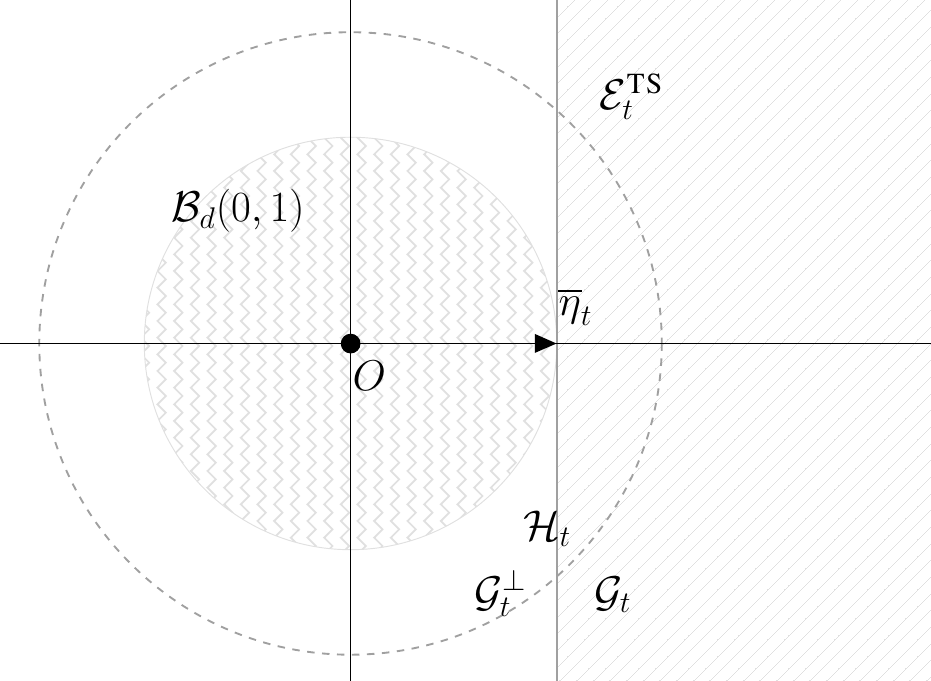}
\end{center}
\caption{Illustration of the probability of selecting an optimistic $\wt\theta_t$.}
\label{fig:optimism.eta}
\end{figure} 
Finally, we show that this property is not affected, up to a second order term, by the high-probability concentration event. It relies on the fact that the chosen confidence level $\delta^\prime = \delta / 4T$ is small compared to the anti-concentration probability $p$ of Def.~\ref{def:ts.exploration}. For sake of simplicity, we assume that $T \geq 1/2p$ which implies that $\delta^\prime \leq p/2$.\\
For any events $A$ and $B$, one has
\begin{align*}
\mathbb{P} (A \cap B) = 1  - \mathbb{P}(A^c \cup B^c) 
\geq \mathbb{P}(A) - \mathbb{P}(B^c)
\end{align*}
Applying the previous inequality to $A := \{J(\wt\theta) \geq J(\theta^\star) \}$ and $B := \{\wt\theta \in \calE^\ts_t \}$ where $\calE^\ts_t = \{ \theta \in \mathbb{R}^d \hspace{1mm} | \hspace{1mm} \|\theta - \wh\theta_t\|_{V_t} \leq \gamma_t(\delta^\prime) \}$ leads to
\begin{align*}
\mathbb{P} ( \wt\theta_t \in \Theta^\opt \cap \calE^\ts_t | \mathcal{F}_t, \hat{E}_t ) \geq p - \delta^\prime \geq p/2
\end{align*}

\end{proof}

We now make use of the fact that the probability of being optimistic is constant to control the regret.  Let $g(\wt\theta_t)$ be an arbitrary non-negative function of $\wt\theta_t$, then we can write the full expectation as
\begin{align*}
\E[g(\wt\theta_t) | \F_t, \wh{E}_t ] &\geq \E\big[g(\wt\theta_t) | \wt\theta_t \in \Theta_t^{\opt}, \F_t, \wh{E}_t \big] \mathbb{P}\big( \wt\theta_t \in \Theta_t^{\opt} \big) \\
&\geq  \E\big[g(\wt\theta_t) | \wt\theta_t \in  \Theta_t^{\opt}, \F_t, \wh{E}_t \big] p/2.
\end{align*}

Setting $g(\wt\theta) = 2 \gamma_t(\delta^\prime) \|x^\star(\wt\theta)\|_{V_t^{-1}}$ one obtains:
\begin{equation*}
\begin{aligned}
R_t^\ts &\leq 4 \gamma_t(\delta^\prime)/p \hspace{1mm}  \E \big[ \|x^\star(\wt\theta)\|_{V_t^{-1}}  | \F_t , \wh{E}_t \big] \mathbb{P}(\wh{E}_t),\\
&\leq 4 \gamma_t(\delta^\prime)/p \hspace{1mm} \E \big[ \|x^\star(\wt\theta)\|_{V_t^{-1}} \I\{\wh{E}_t\} | \F_t \big],
\end{aligned}
\end{equation*}
where we reintegrate the event $\wh{E}_t$ in the expectation in the last inequality. $2/p$ can be interpreted as the expected time between any two optimistic samples. Finally, we can use Azuma's inequality to obtain the final bound with probability at least $1-\delta/2$
\begin{align*}
R^\ts(T) &\leq \frac{4 \gamma_T(\delta^\prime)}{p} \Big( \sum_{t=1}^T \|x_t\|_{V_t^{-1}} + \sqrt{\frac{8 T } {\lambda} \log \frac{4}{\delta} }\Big),
\end{align*}
where $x_t$ is the optimal arm $x^\star(\wt\theta_t)$ selected by \ts. The proof is concluded using Cauchy-Schwarz and Prop.~\ref{p:self_normalized_determinant} to bound $R^\ts(T)$ and Prop.~\ref{p:concentration} to bound $R^\rls(T)$.



\section{Extensions}\label{s:extensions}

We provide an alternative proof for \ts in LB, leveraging the properties of the optimal value function, the properties of the \rls estimate and the \textit{concentration/anti-concentration} property of the sampling. As the functioning of the proof does not rely on the \textit{specific} shape of the $J$ function, we can readily apply it to similar, yet, more general linear problems. We present here two extensions: the regularized linear optimization problem and the generalized linear bandit.\\

\textbf{Regularized linear optimization.} Our proof holds for any arm set $\X$ and the corresponding constrained optimization problem $\max_{x\in\X} x^\transp \theta^\star$. Similarly, we can apply it to any regularized linear optimization problem $\max_{x\in\Re^d} f_{\mu,c}(x; \theta)$, with $f_{\mu,c}(x; \theta) = x^\transp \theta+ \mu c(x)$, where $\mu$ is a constant and $c(x)$ is an arbitrary penalty function of $x$ (e.g., norm-regularization). 
While there always exists a set of constraints (corresponding to a set of arms $\X_{c,\mu,\theta}$) such that the solution to the constrained and regularized problems coincides, such mapping is often unknown (e.g., $c(x) = \|x\|_1$) and thus \ts cannot be run on $\X_{c,\mu,\theta}$ but we need to directly deal with the regularized problem (i.e., sampling $\wt\theta_t$ and pulling arm $x_t = \arg\max_{x}f_{\mu,c}(x; \wt\theta_t)$). In this case, it can be seen that the three main steps of our proof still hold. In fact (see App.~\ref{sec:app_regularized_linear_optimization}), \textbf{1)} $J(\theta)$ is convex, \textbf{2)} the gradient of $J(\theta)$ corresponds to the optimal arm $x^*(\theta)$, \textbf{3)} Lemma~\ref{le:probability_optimistic} holds unchanged since it relies on the convexity of $J(\theta)$ and the \ts distribution $\D^{\ts}$ is the same. As a result, the regret bound follows. On the other hand, the original proof by~\citet{agrawal2012thompson} could be less readily applied to this case. First notice that the mapping from $\mu$ and $c(x)$ to the constrained set $\X_{c,\mu,\theta^\star}$ requires the unknown parameter $\theta^\star$. This means that if we pass from the regularized problem to the constrained problem at each time step $t$, we would be working on a set  $\X_{c,\mu,\wt\theta_t}$ which keeps changing over time. While \citet{agrawal2012thompson} study the contextual bandit problem where $\X_t$ changes arbitrarily over time, in this case $\X_t$ would change in response to $\wt\theta_t$ itself (i.e., it would not be available in advance) and the analysis would bound the per-step regret $r_t = \max_{x\in\X_{c,\mu,\wt\theta_t}} x^\transp \theta^\star - x_t^\transp \theta$, which does not correspond to the desired regret on $f_{\mu,c}$ (the true optimal arm $x^\star(\theta^\star)$ may not even be in $\X_{c,\mu,\wt\theta_t}$). Alternatively, we need to formulate a suitable definition of saturated and unsaturated arms for $f_{\mu,c}(x; \theta)$, which does not seem trivial and it may require developing a more \textit{ad-hoc} analysis. \\


\textbf{Generalized linear bandit.} Another interesting extension is the generalized linear bandit (GLM) problem of~\cite{filippi2010parametric}. In this setting, the reward associated to arm $x \in \mathcal{X}$ is no longer drawn from the linear regression model but is generated as  $r(x) = \mu(x^\transp \theta^\star) + \xi$, where $\mu$ is the so-called \textit{link function},
 $\theta^\star\in\Re^d$ is a fixed but unknown parameter vector and $\xi$ is a random zero-mean noise. One of the major advantage of this setting is that it encompasses \textit{logistic regression}. It can model the case when the reward is in $[0,1]$ and thus became very popular in recommender system where the reward represents the probability of click.\\
 Formally, the value of an arm $x\in\X$ is evaluated according to its expected reward $\mu(x^\transp\theta^\star)$ and for any parameter $\theta\in\Re^d$ we denote the optimal arm and its optimal value as
\begin{equation}\label{eq:optimal_arm_definition.glm}
x^{\star}(\theta) = \arg \max_{x \in \X} \mu(x^\transp \theta), \quad\quad J^{\text{GLM}}(\theta) = \sup_{x \in \mathcal{X}} \mu(x^\transp \theta).
\end{equation}
Then $x^\star = x^\star(\theta^\star)$ is the optimal arm associated with the true parameter $\theta^\star$ and $J^{GLM}(\theta^\star)$ its optimal value. 
At each step $t$, a learner chooses an arm $x_t \in \X$ using all the information observed so far (i.e., sequence of arms and rewards) but without knowing $\theta^\star$ and $x^\star$. At step $t$, the learner suffers an \textit{instantaneous regret} corresponding to the difference between the expected rewards of the optimal arm $x^\star$ and the arm $x_t$ played at time $t$. The objective of the learner is to minimize the \textit{cumulative regret} up to a terminal step $T$,
\begin{equation}\label{eq:cumulative_regret_definition.glm}
R^{\text{GLM}}(T) = \sum_{t=1}^T \big( \mu(x^{\star,\transp} \theta^\star ) - \mu(x_t^\transp \theta^\star)\big).
\end{equation}
Similarly to the regularized optimization problem, a regret bound can be derived for the GLM problem using the same line of proof that we use for LB. It first relies on the fact that, under suitable assumptions about the link function $\mu$, consistent estimates are available for $\theta^\star$ together with high probability confidence ellipsoids. Then, we can show (see App.~\ref{sec:app_generalized_linear_bandit}) that the GLM optimal value function $J^\text{GLM}(\theta)$ is related to the LB optimal value function $J(\theta)$ as $J^\text{GLM}(\theta) = \mu(J(\theta))$. Finally, by assumption the link function $\mu$ is Lipschitz (with constant $L$) and its first order derivative is lower bounded $\mu^\prime(\theta) \geq c$. Therefore, one can bound the regret as $R^\ts_t = J^\text{GLM}(\theta^\star) - J^{\text{GLM}}(\wt\theta_t) \leq \max(c,L) \big( J(\theta^\star) - J(\wt\theta_t) \big)$ and apply directly the line of proof of LB. We provide the proof in App.~\ref{sec:app_generalized_linear_bandit}.\\

\textbf{Other extensions.} To go further, we can generalize our proof to the other convex optimization problems $\max_{x\in\X} f(x,\theta)$, with linear observations (i.e., $y = x^\transp \theta + \xi$). If $f(x,\theta)$ is convex in $\theta$, then $J(\theta)$ is convex as well, thus enabling the possibility to apply our line of proof. More precisely, the gradient of $J$ to the arms played by \ts should be related (step 2, Lem.~\ref{lem:gradient.optimal.arm}) and the on-policy prediction error $R^{\rls}$ measured w.r.t.\ $f$ should be bounded (Prop.~\ref{p:concentration}). Whenever these properties are satisfied, the regret result follows. Notice that while the original proof by~\citet{agrawal2012thompson} may be extended to cover some of these problems, its requirements are slightly stronger. In fact, the definition of saturated and unsaturated arms relies on the fact that $f(x,\wh\theta_n)$ concentrates to $f(x,\theta)$ for any $x$, while in our case, we only need to bound $R^{\rls}$, which corresponds to an \textit{on-policy} error, where prediction errors are measured \textit{on} the specific arms selected by the algorithm. While this advantage may appear abstract, let consider the reinforcement learning case, where $f(x,\theta)$ is the value function of a policy $x$ in an environment $\theta$. In this case, $f(x,\theta^\star)$ may actually be unbounded for some $x$ (i.e., the policy $x$ does not control the system) and the definition of saturated/unsaturated arms could not be easily adjusted. This suggests that our proof could enable covering special RL cases as well.
Finally, we remark that defining \ts as a randomized algorithm and using convex geometry arguments in its analysis bears a strong resemblance with follow-the-pertubed-leader algorithm and  its regret analysis in adversarial linear bandit~\citep{abernethy2015fighting}, suggesting that the two approaches may be strongly related.


\section{Discussion}\label{s:discussion}

We developed an alternative proof for \ts in LB with novel insights on the core elements of the algorithm (\textit{optimism}) and the structure of the problem (\textit{support function} $J(\theta)$). There are a number of possible applications of our results and future directions of investigation. The main open question is whether or not oversampling is needed to guarantee a $\sqrt{T}$ regret bound for TS. Since this worsen the bound by $\sqrt{d}$, answering this question could improve the current frequentist bound from $\tilde{O} \big( d^{3/2} \sqrt{T} \big )$ to $\tilde{O} \big( d \sqrt{T} \big)$, thus matching the bound achieved by OFUL. We first present numerical experiments that compare the two versions of \ts and 
exhibit the dependency of the constant w.r.t $d$. Then, we stress why oversampling is needed in the current analysis and discuss how to relax it.\\

\textbf{Numerical experiments.} 
To understand the impact of the oversampling, we compare two instances of the \ts algorithm:
\begin{enumerate}
\item we denote as $FreqTS$ the instance of the algorithm where, at each time step, the parameter is sampled as $\wt\theta_t = \wh\theta_t + \beta_t V_t^{-1/2} \eta$ with $\eta \sim \mathcal{N}(0,I)$, for which we prove a $\tilde{O} \big( d^{3/2} \sqrt{T}\big)$ regret bound,
\item we denote as $BayesTS$ its bayesian counterpart where, at each time step, the parameter is sampled as $\wt\theta_t = \wh\theta_t + V_t^{-1/2} \eta$ with $\eta \sim \mathcal{N}(0,I)$, for which no frequentist regret guarantee exists.
\end{enumerate}
We compute the regret over trajectories of length $T=200000$ for values of $d$ spanning $[0,30]$. The motivation for such long trajectories is that the regret curves exhibit slightly different regimes (w.r.t. $t$). Since we focus on the dependency on $d$, we discard this effect ensuring that each trajectory reaches the asymptotic regime. The parameter $\theta^\star$ is fixed at the beginning of each trajectory as $\theta^\star = (1,0,\dots,0)$. The reason for imposing $\|\theta^\star\| =1$ is to remove the dependency on $d$ in the norm of $\theta^\star$, which affects the regret through the constant $S$ of Asm.~\ref{asm:param.set}. The \rls estimation is initialized as $V_0 = \lambda I$ with $\lambda =1$ and $\wh\theta_0$ is randomly chosen on the unite sphere. Finally, the reward noise sequences $\{ \xi_{t} \}_t$ are generated  i.i.d according to $\xi_t \sim \mathcal{N}(0,I)$.

We present the result on Fig.\ref{fig:TS.comparison}. On the l.h.s, we draw the average and high probability cumulative regret: if the $FreqTS$ algorithm exhibits a $\sqrt{T}$ regret as expected from the theoretical guarantee, the $BayesTS$ algorithm offers even better performance. While the $\sqrt{T}$ shape is preserved, the dependency of the constant in the dimension $d$ is better, as illustrated on the figure on the r.h.s: we draw the final regret $R(T)$ for both algorithms, as a function of the dimensionality $d$ of the problem: again, as expected from the theoretical guarantee, the $FreqTS$ regret scales as $d^{3/2}$, whereas the $BayesTS$ regret scales as $d$, thus matching the bound of OFUL.\\

\begin{figure}[!h]
\vspace{0.4in}
\centering
\setlength{\unitlength}{\textwidth}
\begin{picture}(1,0.5)
\put(-.17,0){\includegraphics[width=1.3\textwidth]{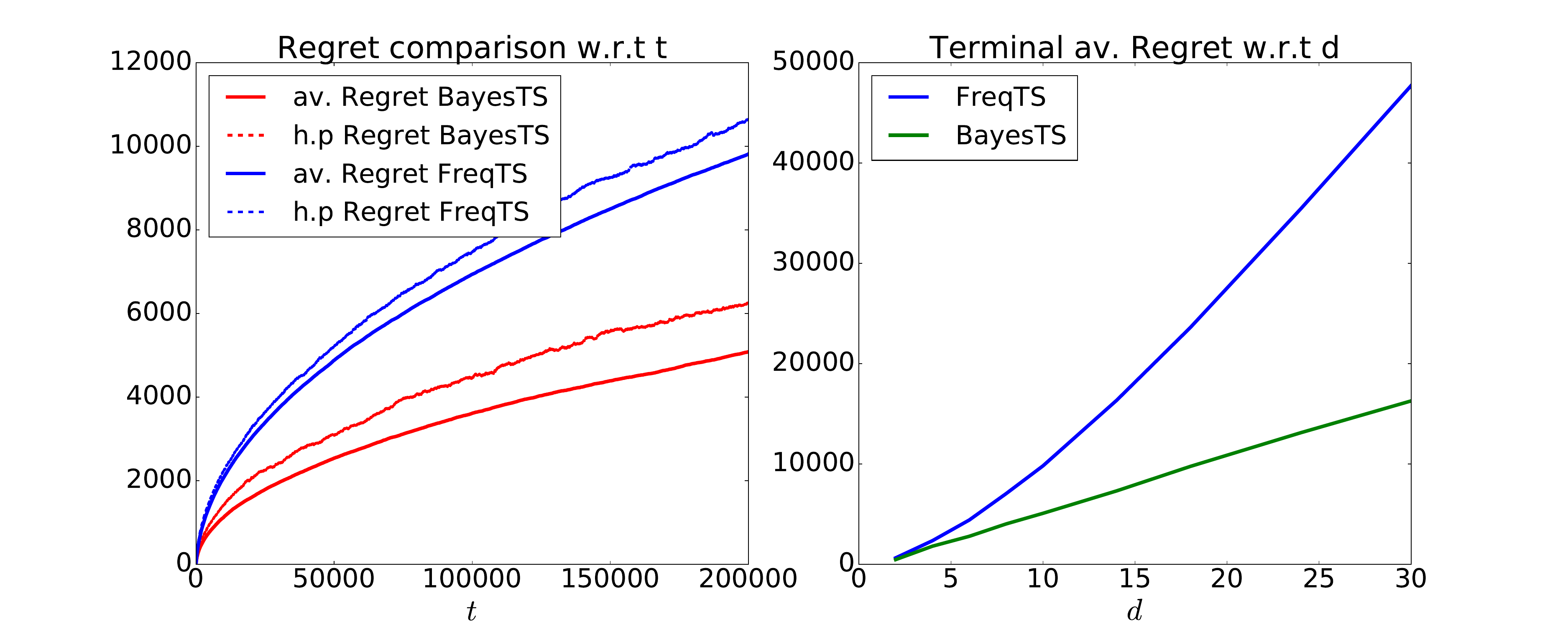}}
\end{picture}
\vspace{-0.3in}
\caption{\small Regret comparison between $BayesTS$ and $FreqTS$ algorithm. \textit{Left:} empirical average and high probability regret w.r.t $t$ for $100$ trajectories with $d=10$. \textit{Right:} empirical average terminal regret $R(T)$ for $T=200000$ w.r.t $d$ for $100$ trajectories.}
\label{fig:TS.comparison}
\vspace{-0.1in}
\end{figure}

The intuition provided by this experiment is twofold: first, it stresses that the $\tilde{O} \big( d^{3/2} \sqrt{T}\big)$ regret bound of the $FreqTS$ algorithm is tight, so a factor $\sqrt{d}$ cannot be removed by a different analysis; secondly, it suggests that no oversampling is needed to guarantee a $\sqrt{T}$ regret and that a $\tilde{O} \big( d \sqrt{T}\big)$ regret bound could be derived for the $BayesTS$ algorithm. \\

\textbf{About optimism and oversampling.} 
As illustrated in Sect.~\ref{sec:geometry}, in the current proof optimistic steps allows to bound the regret of non-optimistic steps. Nonetheless, it can be shown that some non-optimistic steps (even very \textit{pessimistic}!) may indeed be as ``informative'' as optimistic steps and allow reducing the regret as well. Let consider a minor change in the line of proof, anticipating the use of the convexity of $J$, i.e.,
\begin{align*}
R_t^{\ts} &\leq \sup_{\theta\in\calE_t^\ts} \nabla J(\theta^\star)^\transp (\theta^\star-\theta)\I\{E_t\} \\
&\leq  \|\nabla J(\theta^\star)\|_{V_t^{-1}}2 \gamma_t(\delta')\I\{E_t\}.
\end{align*}
If we sample a $\wt\theta$ such that the gradient at it $\nabla J(\wt\theta)$ (i.e., which coincides with the corresponding optimal action $x^\star(\wt\theta)$) has the same $V_{t}^{-1}$-norm as $\nabla J(\theta^\star)$, then we could apply the same reasoning as in the original sketch of the proof and bound the regret of any subsequent step. More formally, we can define the set $\Theta_t^{\grad} = \{ \theta: \|\nabla J(\theta)\|_{V_t^{-1}} \geq \|\nabla J(\theta^\star)\|_{V_t^{-1}} \}$ of parameters that have larger gradient than $\theta^\star$'s. Similar to $\Theta^{\opt}$, if the probability of sampling $\wt\theta$ in $\Theta_t^{\grad}$ is lower-bounded by a constant $p'$, then the proof can be reproduced with exactly the same arguments and result. Even further, we could relax the requirement and define $\Theta_t^{\grad}(\alpha) = \{ \theta: \|\nabla J(\theta)\|_{V_t^{-1}} \geq \alpha\|\nabla J(\theta^\star)\|_{V_t^{-1}} \}$, with $\alpha<1$, which would allow even a bigger probability at the cost of an extra constant factor $\alpha$ in the final regret. 
\begin{figure}[h]
\begin{center}
\includegraphics[width=0.8\textwidth]{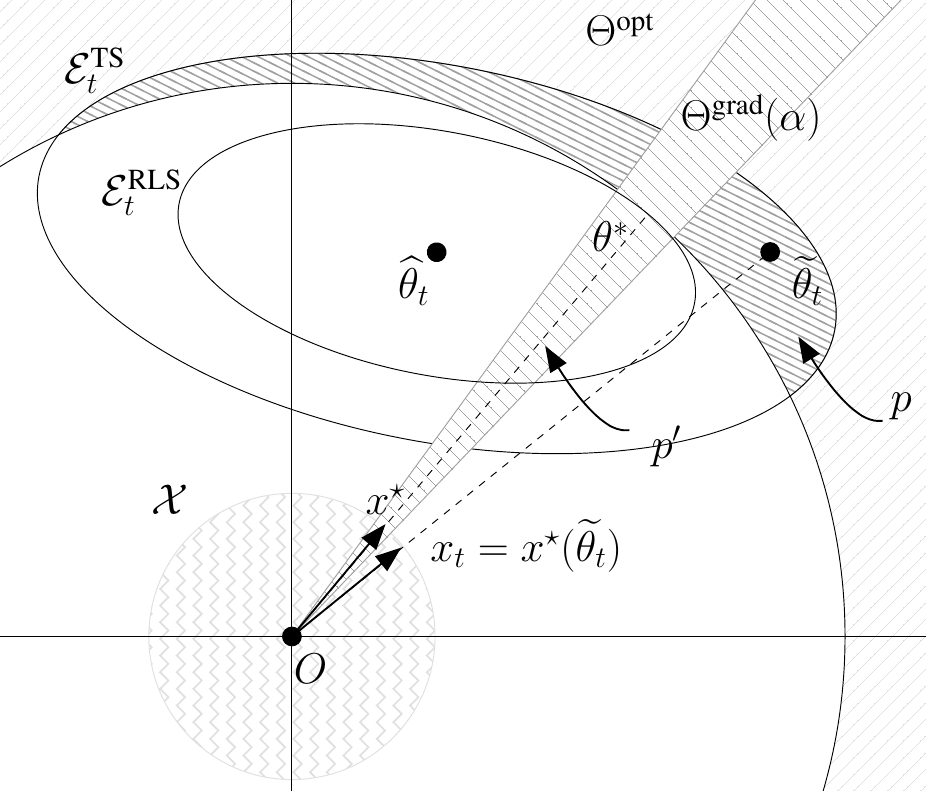}
\end{center}
\vspace{-0.1in}
\caption{\small Illustration of the non-optimistic region that could contribute to reduce the regret.}
\label{fig:illustration2}
\vspace{-0.1in}
\end{figure}
As illustrated in Fig.~\ref{fig:illustration2}, in the case $\X = B_d(0,1)$, $\Theta_t^{\grad}(\alpha)$ corresponds to a cone whose overlap with $\calE^{\ts}$ may actually be even larger than for $\Theta^{\opt}$. This illustration shows that the set of \textit{useful} explorative actions does not necessarily coincide with the set of optimistic parameters and that many more parameters in $\calE^{\ts}$ may contribute to reduce the regret. This may explain the empirical success of \ts and it may suggest that the oversampling by a factor $\sqrt{d}$ to ensure optimism may be a too strong requirement. Finally, we remark that a similar optimistic argument is employed by~\citet{agrawal2013further} in MAB. Nonetheless, in Lemma~2 they prove that the probability of being optimistic increases over time. This may suggest that $\calE^{\ts}$ needs to be only a \textit{constant} fraction bigger than $\calE^{\rls}$, since the initial small probability of being optimistic would tend to a constant (or even to 1) later on during the learning process. Whether this argument holds and how to prove it remains an open question.\\

{\small
\textbf{Acknowledgement} This research is supported in part by a grant from CPER Nord-Pas de Calais/FEDER DATA Advanced data science and technologies 2015-2020, CRIStAL (Centre de Recherche en Informatique et Automatique de Lille), and the French National Research Agency (ANR) under project ExTra-Learn n.ANR-14-CE24-0010-01.
}

\newpage
\begin{small}
\bibliography{biblio_bandit}
\bibliographystyle{plainnat}
\end{small}

\newpage
\onecolumn
\appendix


\section{Examples of \ts distributions}\label{sec:app_examples}

\textit{Example 1: Uniform distribution $\eta \sim \mathcal{U}_{B_d(0,\sqrt{d})}$.}
The uniform distribution satisfies the concentration property with constants $c = 1$ and $c^\prime = \frac{e}{d}$ by definition. Since the set $\{ \eta | u^\transp \eta \geq 1\} \cap B_d(0,\sqrt{d})$ is an hyper-spherical cap for any direction $u$ of $\mathbb{R}^d$, the the anti-concentration property is satisfied provided that the ratio between the volume of an hyper-spherical cap of height $\sqrt{d} -1$ and the volume of the ball of radius $\sqrt{d}$ is constant (i.e., independent from $d$). Using standard geometric results (see Prop.~\ref{p:volume}), one has that for any vector $\|u\| = 1$
\begin{equation}
\mathbb{P} ( u^\transp \eta \geq 1) = \frac{1}{2} I_{1-\frac{1}{d}} \Big(\frac{d+1}{2},\frac{1}{2}\Big),
\label{eq:hyperspherical_cap_beta_function}
\end{equation}
where $I_x(a,b)$ is the incomplete regularized beta function. In Prop.~\ref{p:hyperspherical.cap} we prove that 
\begin{equation*}
I_{1-\frac{1}{d}} \Big(\frac{d+1}{2},\frac{1}{2}\Big) \geq \frac{1}{8 \sqrt{3 \pi}},
\end{equation*}
and hence we obtain $p = \frac{1}{16 \sqrt{3 \pi}}$.{\hfill $\qed$}.

\textit{Example 2: Gaussian case $\eta \sim \mathcal{N}(0,I_d)$.}
The concentration property comes directly from the Chernoff bound for standard Gaussian random variable together with union bound argument. For any $\alpha > 0$, we have
\begin{equation*}
\mathbb{P}( \| \eta \| \leq \alpha \sqrt{d} ) \geq \mathbb{P}(\forall 1 \leq i \leq d, \hspace{1mm} | \eta_i | \leq \alpha) \geq 1 - d \mathbb{P} (| \eta_i | \geq \alpha).
\end{equation*}
Standard concentration inequality for Gaussian random variable gives, $\forall \alpha > 0$, 
\begin{equation*}
\mathbb{P} (| \eta_i | \geq \alpha)  \leq 2 e^{ - \alpha^2/2}.
\end{equation*}
Plugging everything together with $\alpha = \sqrt{ 2 \log \frac{ 2 d}{\delta} }$ gives the desired result with $c = c^\prime = 2$. Let $\eta_i$ be the $i$-th component of $\eta$ for any $1 \leq i \leq d$. Then $\eta_i \sim \mathcal{N}(0,1)$. Since $\eta$ is rotationally invariant, for any direction $u$ of $\mathbb{R}^d$ and an appropriate choice of basis, we have $\mathbb{P}(u^\transp \eta \geq 1)  \geq \mathbb{P}( \eta_1 \geq 1)$. From standard Gaussian properties (see Thm 2 of~\cite{chang2011chernoff}) we have  
\begin{equation*}
\mathbb{P}( \eta_1 \geq 1) = \frac{1}{2} \text{erfc} \left( \frac{1}{\sqrt{2}} \right)  \geq \frac{1}{4 \sqrt{e \pi}}
\end{equation*}
which ensures the anti-concentration property with $p = \frac{1}{4 \sqrt{e \pi}}$.{\hfill $\qed$}


\section{Properties of convex function}\label{sec:app_convex}

\begin{proposition}
\label{pr:maximum_convex_closed_set}
Let $f: \mathbb{R}^{d} \rightarrow \mathbb{R}$ be a convex function and $C$ be a closed convex subset of $\mathbb{R}^d$. Then, on $C$, there exists a point on the boundary of $C$ that achieves the argmax.
\end{proposition}

\begin{proof}
Let's denote as $int(C)$ and $bound(C)$ the interior and the boundary of the closed convex set $C$ respectively. Assume that $\exists x^\star \in int(C)$ such that $f(x^\star) > f(x)$ for any $x \in bound(C)$ and  $f(x^\star) \geq f(y)$ for any $y \in int(C)$.\\

Then define $y = x^\star + \epsilon (x^\star - x)$ for some $x \in bound(C)$. By definition of the open set $int(C)$, $\exists \epsilon > 0 $ such that $y \in int(C)$. Moreover, $x^\star \in [ y , x]$ e.g. 
\begin{equation*}
x^\star = (1-t) x + t y , \hspace{5mm} t = \frac{1}{1+\epsilon} \in ]0,1[
\end{equation*}
Using the convexity of $f$, one has
\begin{equation*}
\begin{split}
f(x^\star) &\leq (1-t) f(x) + t f(y) < (1-t) f(x^\star) + t f(y) \\
f(x^\star) &< f(y)
\end{split}
\end{equation*}
which is impossible by assumption. 
\end{proof}

\begin{proposition}
\label{pr:increasing_subspace}
Let $f: \mathbb{R}^{d} \rightarrow \mathbb{R}$ be a convex function. Let $B_d(0,1)$ be the unit $d-$dimensional ball and $S_d(0,1)$ the associated unit sphere.\\
Let $x^\star \in S_d(0,1)$ such that $f(x^\star) \geq f(x)$ for all $x \in B_d(0,1)$, and let $\mathcal{H}(x^\star)$ be the hyperplan tangent to $B_d(0,1)$ at the point $x^\star$, which splits $\mathbb{R}^d$ into two complementary subsets $\mathcal{G}(x^\star)$ and $\mathcal{G}^\perp(x^\star)$ defined respectively by
\begin{equation*}
\begin{aligned}
\mathcal{H}(x^\star)&= \{ x \in \mathbb{R}^d \text{ s.t. } x^\transp x^\star = 1 \},\\
 \mathcal{G}(x^\star) &= \{ x \in \mathbb{R}^d \text{ s.t. } x^\transp x^\star \geq 1 \}, \\
  \mathcal{G}(x^\star)^\perp &= \{ x \in \mathbb{R}^d \text{ s.t. } x^\transp x^\star < 1 \}.
 \end{aligned}
\end{equation*}
Then, $\forall y \in \mathcal{G}(x^\star), \hspace{3mm} f(y) \geq f(x^\star)$.
\end{proposition}

\begin{proof}
We first notice that from Proposition~\ref{pr:maximum_convex_closed_set} $x^\star$ is well defined since the maximum is reached on the boundary. The associated subset $\mathcal{G}(x^\star)$ is then
\begin{equation*}
\mathcal{G}(x^\star) := \{ y = x^\star + u , u \in \mathbb{R}^d \hspace{2mm} | \hspace{2mm} u^\transp x^\star \geq 0 \}.
\end{equation*}
We want to show that $f(y) \geq f(x^\star)$ for any $y \in \mathcal{G}(x^\star)$. We introduce the increasing sequence of subsets
\begin{equation*}
\mathcal{G}_n = \left\{ y = x^\star + u, u \in \mathbb{R}^d \hspace{2mm} | \hspace{2mm} u^\transp x^\star \geq \frac{\|u\|}{2(n-1)} \right \}, \hspace{3mm} n \geq 2.
\end{equation*}
For any $y = x^\star + u$ in $\mathcal{G}_n$, we associate  
\begin{equation*}
\begin{split}
x &= x^\star - \frac{1}{2 (n-1)} \frac{u}{\|u\|}.
\end{split}
\end{equation*}
By definition of $y$ (and hence $u$), we have
\begin{equation*}
\begin{split}
\|x\|^2 &= 1 + \frac{1}{2(n-1)}^2 - \frac{1}{2(n-1) \|u\|} u^\transp x^\star \\ 
&= 1 + \frac{1}{2(n-1)} \left[ \frac{1}{2(n-1)} - \frac{u^\transp}{\|u\|}x^\star \right] \\
&\leq 1,
\end{split}
\end{equation*}
which means that $x \in \ball_d(0,1)$.
Moreover let $t = [2(n-1)\|u\| +1 ]^{-1}$, $t \in ]0,1[$ one has $x^\star = (1-t)x + ty$. Since $x \in \ball_d(0,1)$ then
\begin{equation*}
\begin{split}
f(x^\star) &\leq (1-t) f(x) + t f(y) \\
&\leq (1-t) f(x^\star) + t f(y) \\
\Rightarrow & f(x^\star) \leq f(y).
\end{split}
\end{equation*}
Since the statement of the proposition holds for any $\mathcal{G}_n$, then we obtain the desired result for $\mathcal{G}(x^\star)$ by continuity of $f$.
Let $y \in \mathcal{G}(x^\star)$, $y = x^\star + u$. If $u^\transp x^\star > 0$, then $\exists n \geq 2$ such that $y \in \mathcal{G}_n$ and the proposition is satisfied.
Otherwise, if $u^\transp x^\star = 0$, we introduce the sequences $\{u_n\}$ and $\{ y_n \}$ defined as:
\begin{equation*}
\begin{split}
u_n &= u + \frac{\|u\|}{\sqrt{1 - \frac{1}{2(n-1)}^2 }} \frac{x^\star}{2(n-1)} \\
 &= u + \frac{\|u_n\|}{2(n-1)} x^\star, \\
y_n &= x^\star + u_n.
\end{split}
\end{equation*}
By construction, $y_n \in \mathcal{G}_n$ and $y_n \rightarrow y$ as $n \rightarrow \infty$. Since the $f(y_n) \geq f(x^\star)$ for any $n \geq 2$ we obtain the desired result taking the limit since $f$ is continuous as a convex function on $\mathbb{R}^d$.
\end{proof}

\begin{theorem}[A.D. Alexandrov]
\label{th:alexandrov_theorem}
Let $f: \mathbb{R}^d \rightarrow \mathbb{R}$ be a convex function, then it is twice differentiable almost everywhere with respect to the Lebesgue's measure.
\end{theorem}

\begin{proof}
This result is an extension of the Rademacher's theorem for convex functions. A proof can be found in ~\cite{niculescu2006convex}, theorem 3.11.2.
\end{proof}

\section{Properties of support function (proof of Proposition~\ref{pr:J_support_function} and Lemma~\ref{lem:gradient.optimal.arm})}\label{sec:app_support}

We study the \textit{support function} of a set $C$, which is a function $f_C : \mathbb{R}^d \rightarrow \mathbb{R}$ such that
\begin{equation}
f_C(\theta) = \sup_{x \in C} x^\transp \theta
\label{eq:support_function_definition}
\end{equation}
Those functions are at the core of convex geometry analysis. 
\begin{proposition}
\label{pr:support_function_generic_properties}
Let $C \subset \mathbb{R}^d$ be a non-empty compact set and $f_C$ the associated support function. Then,
\begin{enumerate}
\item $f_C$ is real-valued and $\sup_{x \in C} x^\transp \theta$ is attained in $C$,
\item $f_C$ is convex,
\item $f_C$ is continuous on $\mathbb{R}^d$ and twice differentiable almost everywhere with respect to the Lebesgue's measure.
\end{enumerate}
\end{proposition}

\begin{proof}
\begin{enumerate}
\item This comes directly from the compactness of $C$: since $C$ is bounded, the support function is real-valued and since $C$ is closed, the supremum is attained in $C$,
\item Let $\theta_1$, $\theta_2$ two vectors of $\mathbb{R}^d$, and $t \in (0,1)$. By definition of the supremum, since $f_C$ is real-valued:
\begin{equation*}
f_C(t \theta_1 + (1-t)\theta_2) = \sup_{x \in C} \big( t x^\transp \theta_1 + (1-t) x^\transp \theta_2 \big) \leq t \sup_{x \in C} x^\transp \theta_1 + (1-t) \sup_{x \in C} x^\transp \theta_2
\end{equation*}
\item The continuity is a consequence of the convexity of $f_C$ on the open convex set $\mathbb{R}^d$ and the second order differentiability comes from Alexandrov's theorem \ref{th:alexandrov_theorem}.
\end{enumerate}

\end{proof}

\begin{proposition}
\label{pr:gradient_support_function}
Let $x(\theta) \in \arg \max_{x \in C} x^\transp \theta$, denote as $\nabla f_C(\theta)$ and $\partial f_C(\theta)$ the gradient (when it is uniquely defined) and the sub-gradient of $f_C$ in $\theta \in \mathbb{R}^d$. Then,
\begin{enumerate}
\item for all $\theta \in \mathbb{R}^d$, $x(\theta) \in \partial f_C(\theta)$,
\item their exists a null set $\mathcal{N}$ with respect to the Lebesgue's measure such that $x(\theta) = \nabla f_C (\theta)$ for all $\theta \in \mathbb{R}^d \setminus \mathcal{N}$,
\item equivalentely, $x(\theta) = \nabla f_C(\theta)$ almost everywhere.
\end{enumerate}
\end{proposition}

\begin{proof}
Thanks to proposition \ref{pr:support_function_generic_properties}, we know that the maximum is attained in $x(\theta) \in C$. Moreover, Alexandrov's theorem guarantee that $\mathcal{N}$ is a null-set. Since the sub-gradient is reduced to a singleton where the function is differentiable e.g. $\partial f_C(\theta) = \{ \nabla f_C(\theta) \}$ for all $\theta \in \mathbb{R}^d \setminus \mathcal{N}$, one just need to show to $x(\theta) \in \partial f_C(\theta)$ for all $\theta \in \mathbb{R}^d$.\\
Since $f_C(\theta) = \max_{x \in C} x^\transp \theta$, their exist at least one $x(\theta) \in C$ for which the maximum is attained i.e. 
$x(\theta)^\transp \theta = f_C(\theta)$. Moreover, for any $\bar{\theta} \in \mathbb{R}^d$, $f_C(\bar{\theta}) \geq x(\theta)^\transp \bar{\theta}$ by definition. Therefore,
\begin{equation*}
\begin{split}
&f_C(\bar{\theta}) - x(\theta)^\transp \bar{\theta} \geq 0 :=  f_C(\theta) - x(\theta)^\transp \theta \\
&f_C(\bar{\theta}) \geq f_C(\theta) + x(\theta) ^\transp \left( \bar{\theta} - \theta \right), \hspace{3mm} \forall \bar{\theta} \in \mathbb{R}^d
\end{split}
\end{equation*}
which is the definition of the sub-gradient.

\end{proof}


\section{Regret Proofs}\label{sec:app_proofs}

We collect here the main tools that we need for the proof.
We first recall the Azuma's concentration inequality for super-martingales.
\begin{proposition}\label{th:azuma}
If a super-martingale $(Y_t)_{t \geq 0}$ corresponding to a filtration $\mathcal{F}_{t}$ satisfies $|Y_t - Y_{t-1}| < c_t$ for some constant $c_t$ for all $t=1,\dots,T$ then for any $\alpha > 0$,
\begin{equation*}
\mathbb{P}(Y_T - Y_0 \geq \alpha) \leq 2 e^{-\frac{\alpha^2}{2\sum_{t=1}^T c_t^2}}
\end{equation*}
\end{proposition}


\begin{proof}[Proof of Lemma~\ref{le:high_proba_concentration}]
We first bound the two events separately.

\textbf{Bounding $\wh{E}_T$.} 
This bound is a straightforward application of Proposition~\ref{p:concentration} together with a union bound argument. Let $\delta' = \delta/(4T)$, then
\begin{equation*}
\begin{split}
\forall 1 \leq t \leq T, \hspace{5mm}& \mathbb{P} \left(  \|\wh{\theta}_{t} - \theta^\star\|_{V_t} \leq \beta_t(\delta')  \right) \geq 1 - \delta' \\
\text{from union bound}, \hspace{5mm} & \mathbb{P} \left( \bigcap_{t=1}^T \left\{ \|\wh{\theta}_{t} - \theta^\star\|_{V_t} \leq \beta_t(\delta')\right\} \right) \geq 1 - \sum_{t=1}^T \mathbb{P} \left(  \|\wh{\theta}_{t} - \theta^\star\|_{V_t} \geq \beta_t(\delta')  \right) \\
\Rightarrow \hspace{5mm} & \mathbb{P} \left( \bigcap_{t=1}^T \left\{ \|\wh{\theta}_{t} - \theta^\star\|_{V_t} \leq \beta_t(\delta')\right\} \right) \geq 1 - \sum_{t=1}^T \delta' \\
\Rightarrow \hspace{5mm} & \mathbb{P} \big(\wh{E}_T \big) \geq 1 - T\delta' = 1-\frac{\delta}{4}.
\end{split}
\end{equation*}

\textbf{Bounding $\wt{E}_T$.} 
This bound comes directly from the concentration property of the \ts sampling distribution.
From the expression of $\wt{\theta}_t = \wh{\theta}_t + \beta_t(\delta^\prime) V_t^{-1/2} \eta_t$ where $\eta_t$ is drawn i.i.d. from $\distro$, we have
\begin{equation*}
\forall 1 \leq t \leq T, \hspace{5mm} \mathbb{P} \left(  \|\wt{\theta}_{t} - \wh{\theta}_{t}\|_{V_t} \leq \beta_t(\delta^\prime)  \sqrt{c d \log \frac{c^\prime  d}{\delta^\prime}} \right) = \mathbb{P} \left(  \|\eta_t\| \leq   \sqrt{c d \log \frac{c^\prime d}{\delta^\prime}} \right).
\end{equation*}
Then from Definition~\ref{def:ts.exploration}, we have
\begin{equation*}
\mathbb{P} \left(  \|\eta_t\| \leq   \sqrt{c d \log \frac{c^\prime d}{\delta^\prime}} \right) \geq 1 - \delta^\prime.
\end{equation*}

As before, a union bound over the two bounds ensures that
\begin{equation*}
\mathbb{P}(\wt{E}_T ) \geq 1 - T\delta^\prime = 1 - \frac{\delta}{4}.
\end{equation*}

Finally, a union bound argument between the two terms leads to 
\begin{equation*}
\mathbb{P}(\wh{E}_T \cap \wt{E}_T) \geq 1 - \frac{\delta}{2}.
\end{equation*}
\end{proof}


\begin{proof}[Proof of Theorem~\ref{th:regret_ts}]

We first bound the two regret terms $R^\ts(T)$ and $R^\rls(T)$.

\textbf{Bound on $R^\ts(T)$.}
We collect the bounds on each term $R_t^\ts$ and obtain
\begin{equation}\label{eq:R11_bound}
R^\ts(T) \leq \sum_{t=1}^T R_{t}^{\ts}\I\{E_t\} \leq  \frac{4 \gamma_T(\delta^\prime)}{p} \sum_{t=1}^T\E\big[ \|x^\star(\wt\theta)\|_{V_{t}^{-1}} | \F_{t} \big].
\end{equation}
Since this term contains an expectation, we cannot directly apply Proposition~\ref{p:self_normalized_determinant} and we first need to rewrite to the total regret $R^{\ts}(T)$ as
\begin{align*}
R^{\ts}(T) &\leq \frac{4 \gamma_T(\delta^\prime)}{p} \bigg( \sum_{t=1}^T\|x_t\|_{V_{t}^{-1}} + \underbrace{\sum_{t=1}^T\Big(\E\big[ \|x^\star(\wt\theta)\|_{V_{t}^{-1}} | \F_{t} \big] -\|x_t\|_{V_{t}^{-1}}  \Big)}_{R_2^{\ts}} \bigg).
\end{align*}
From Prop.~\ref{p:self_normalized_determinant}, the first term is bounded as,
\begin{equation*}
\sum_{t=1}^T\|x_t\|_{V_{t}^{-1}} \leq \sqrt{T} \bigg( \sum_{t=1}^T\|x_t\|^2_{V_{t}^{-1}} \bigg)^{1/2} \leq \sqrt{ 2 T d \log \big( 1 + \frac{T}{\lambda} \big) }.
\end{equation*} 
We now proceed applying Azuma inequality \ref{th:azuma} to the second term which is a martingale by construction. Under assumption \ref{asm:arm.set}, $ \| x_t \| \leq 1$ for all $t \geq 1$, so since $V_t^{-1} \leq \frac{1}{\lambda} I$ one gets,%
\begin{equation*}
\E\big[ \|x^\star(\wt\theta)\|_{V_{t}^{-1}} | \F_{t} \big] -\|x_t\|_{V_{t}^{-1}} \leq \frac{2}{\sqrt{\lambda}}, \hspace{5mm} a.s.
\end{equation*}
This provides an upper-bound on each element of $R_2^{\ts}$ which holds with probability at least $1 - \frac{\delta}{2}$ as
\begin{equation*}
R_2^{\ts} \leq \sqrt{\frac{8 T } {\lambda} \log \frac{4}{\delta} }. 
\end{equation*}

\textbf{Bound on $R^\rls(T)$.}
The bound on $R^\rls$ is derived as in previous results in~\citep{abbasi2011online,agrawal2012thompson}. We decompose the term in a \textit{sampling prediction error} and a \textit{\rls prediction error} as follow
\begin{equation*}
R^\rls(T) \leq \sum_{t=1}^T  |x_t^\transp (\wt{\theta}_{t} - \wh{\theta}_t)| \I\{E_t\} +  \sum_{t=1}^T  |x_t^\transp (\wh{\theta}_t - \theta^\star)| \I\{E_t\}
\end{equation*}
By definition of the concentration event $E_t$, 
\begin{equation*}
 |x_t^\transp (\wt{\theta}_{t} - \wh{\theta}_t)| \I\{E_t\} \leq \| x_t \|_{V_t^{-1}} \gamma_t(\delta^\prime), \hspace{5mm} 
 |x_t^\transp (\wh{\theta}_t - \theta^\star)| \I\{E_t\} \leq \| x_t \|_{V_t^{-1}} \beta_t(\delta^\prime),
\end{equation*}
so from proposition \ref{p:self_normalized_determinant}, 
\begin{equation}\label{eq:RRLS_bound}
R^\rls(T) \leq \big(\beta_T(\delta^\prime) +  \gamma_T(\delta^\prime) \big) \sqrt{ 2 T d \log \big( 1 + \frac{T}{\lambda} \big) }.
\end{equation}

\textbf{Final bound.}
We finally plug everything together since from lemma~\ref{le:high_proba_concentration} the concentration event holds with probability at least $1- \frac{\delta}{2}$. Using the bound on $R^\ts(T)$ and a union bound argument one obtains the desired result which holds with probability at least $1 - \delta$.
\end{proof}


%



\section{Hyperspherical cap and beta function}\label{sec:app_hyper}

\begin{proposition}\label{p:volume}
Let $V_d(R)$ be the volume of the $d-$dimensional ball of radius $R$ and let $V^{cap}_d(h)$ be the volume of the hyperspherical cap of height $h = R - r > 0$. Then,

\begin{equation*}
V^{cap}_d(h) = \frac{1}{2} V_d(R) I_{1 - (\frac{r}{R})^2}\left(\frac{d+1}{2}, \frac{1}{2} \right)
\end{equation*}
where $I_x(a,b)$ is the incomplete regularized beta function.
\end{proposition}
\begin{proof}
The proof can be found in \cite{li2011concise}.
\end{proof}

\begin{proposition}\label{p:hyperspherical.cap}
Let $I_x(a,b)$ is the incomplete regularized beta function,
\begin{equation*}
\forall d\geq 2, \hspace{3mm} I_{1 - \frac{1}{d}} \left( \frac{d+1}{2}, \frac{1}{2} \right) \geq \frac{1}{8\sqrt{3 \pi}}
\end{equation*}

\end{proposition}

\begin{proof}
The incomplete regularized beta function can be expressed in terms of the beta function $B(a,b)$ and the incomplete beta function $B_x(a,b)$ where
\begin{equation*}
\begin{split}
B_x(a,b) &= \int_{0}^x t^{a-1} (1 - t)^{b-1} dt \\
B(a,b) &= B_1(a,b) \\
I_x(a,b) &= \frac{B_x(a,b)}{B(a,b)}
\end{split}
\end{equation*}
Hence we seek  for a lower bound on $B_{1-\frac{1}{d}} \left( \frac{d+1}{2} , \frac{1}{2} \right)$ and an upper bound for 
$B\left( \frac{d+1}{2} , \frac{1}{2} \right)$.

\begin{enumerate}
\item  Let first find an lower bound for the incomplete beta function. Since $t \rightarrow t^{\frac{d-1}{2}} (1 - t)^{-1/2}$ is positive and increasing on $[0,1]$, for any $d \geq 2$,
\begin{equation*}
\begin{split}
B_{1 - \frac{1}{d}} \left( \frac{d+1}{2}, \frac{1}{2} \right) &\geq \int_{1 - \frac{3}{2d}}^{1-\frac{1}{d}} t^{\frac{d-1}{2}} (1 - t)^{-1/2} dt \\
&\geq \frac{1}{2d} \left( \frac{3}{2d} \right)^{-1/2} \left(1 - \frac{3}{2d}\right)^{\frac{d-1}{2}} \\
&\geq \frac{1}{\sqrt{6d}} \left(1 - \frac{3}{2d}\right)^{\frac{d-1}{2}} \\ 
&\geq \frac{1}{\sqrt{6d}} \left(1 - \frac{3}{2d}\right)^{\frac{d}{2}}
\end{split}
\end{equation*}
From the increasing property of $x \rightarrow (1 - \frac{ \alpha}{x})^x$ for any $\alpha < 1$ the sequence $\left\{ \big(1 - \frac{3}{2d}\big)^{\frac{d}{2}} \right\}_{d \geq 2}$ is increasing and 
\begin{equation*}
B_{1 - \frac{1}{d}} \left( \frac{d+1}{2}, \frac{1}{2} \right) \geq \frac{1}{\sqrt{6 d}} \left(1 - \frac{3}{2 \times 2}\right)^{\frac{2}{2}} = \frac{1}{4 \sqrt{6d}}
\end{equation*}

\item Now we seek for an upper bound for $B\left( \frac{d+1}{2} , \frac{1}{2} \right)$. Since $B(a,b) = \frac{\Gamma(a) \Gamma(b)}{\Gamma(a+b)}$ one has:
\begin{equation*}
B\left( \frac{d+1}{2} , \frac{1}{2} \right) = \frac{\Gamma\left( \frac{1}{2} \right) \Gamma \left( \frac{d+1}{2} \right) }{\Gamma \left( \frac{d}{2} +1  \right)} = \sqrt{\pi} \frac{\Gamma \left( \frac{d+1}{2} \right) }{\Gamma \left( \frac{d}{2} +1  \right)}
\end{equation*}

From \cite{chen2005proof} we have the following inequalities for the gamma function $\forall n \geq 1$:
\begin{equation*}
\begin{split}
\frac{\Gamma (n +1/2)}{\Gamma(n+1)} &\leq ( n + 1/4)^{-1/2} \\
\frac{\Gamma(n + 1/2)}{\Gamma(n+1)} &\geq ( n + 4/\pi -1 )^{-1/2}
\end{split}
\end{equation*}
Together with $\Gamma(x+1) = x \Gamma(x)$ and treating separately cases where $d$ is even or not, one gets $\forall d \geq 2$
\begin{equation*}
\frac{\Gamma \left( \frac{d+1}{2} \right) }{\Gamma \left( \frac{d}{2} +1  \right)} \leq \sqrt{\frac{2}{d}}
\end{equation*}

\item Using the obtained upper and lower bound we get:
\begin{equation*}
I_{1 - \frac{1}{d}} \left( \frac{d+1}{2}, \frac{1}{2} \right) \geq \frac{\sqrt{d}}{\sqrt{2 \pi} \times 4 \sqrt{6d}} \geq \frac{1}{8\sqrt{3 \pi}}
\end{equation*}

\end{enumerate}

\end{proof}


\section{Generalized Linear Bandit}\label{sec:app_generalized_linear_bandit}

We present here how to apply our derivation to the generalized linear bandit (GLM) problem of~\cite{filippi2010parametric}. The regret bound is obtained by basically showing that the GLM problem can be reduced to studying the linear case.

\textbf{The setting.}
Let $\X \subset \Re^d$ be an arbitrary (finite or infinite) set of arms. Every time an arm $x\in\X$ is pulled, a reward is generated as $r(x) = \mu(x^\transp \theta^\star) + \xi$, where $\mu$ is the so-called \textit{link function},
%
%
 $\theta^\star\in\Re^d$ is a fixed but unknown parameter vector and $\xi$ is a random zero-mean noise. The value of an arm $x\in\X$ is evaluated according to its expected reward $\mu(x^\transp\theta^\star)$ and for any parameter $\theta\in\Re^d$ we denote the optimal arm and its optimal value as
\begin{equation}\label{eq:optimal_arm_definition.glm}
x^{\star}(\theta) = \arg \max_{x \in \X} \mu(x^\transp \theta), \quad\quad J^{\text{GLM}}(\theta) = \sup_{x \in \mathcal{X}} \mu(x^\transp \theta).
\end{equation}
Then $x^\star = x^\star(\theta^\star)$ is the optimal arm associated with the true parameter $\theta^\star$ and $J^{GLM}(\theta^\star)$ its optimal value. 
At each step $t$, a learner chooses an arm $x_t \in \X$ using all the information observed so far (i.e., sequence of arms and rewards) but without knowing $\theta^\star$ and $x^\star$. At step $t$, the learner suffers an \textit{instantaneous regret} corresponding to the difference between the expected rewards of the optimal arm $x^\star$ and the arm $x_t$ played at time $t$. The objective of the learner is to minimize the \textit{cumulative regret} up to a finite step $T$,
\begin{equation}\label{eq:cumulative_regret_definition.glm}
R^{\text{GLM}}(T) = \sum_{t=1}^T \big( \mu(x^{\star,\transp} \theta^\star ) - \mu(x_t^\transp \theta^\star)\big).
\end{equation}

\textbf{Assumptions.}
The assumptions associated with this more general problem are the same as in the linear bandit problem plus one regarding the link function. Formally, we require assumption~\ref{asm:arm.set}, ~\ref{asm:param.set} and~\ref{asm:subgaussian} and add:
\begin{assumption}[link function]\label{asm:link.function}
The link function $\mu : \mathbb{R} \rightarrow \mathbb{R}$ is continuously differentiable, Lipschitz with constant $k_\mu$ and such that $c_\mu = \inf_{\theta \in \mathbb{R}^d, x \in \mathcal{X}} (x^\transp \theta) > 0$.
\end{assumption}

\textbf{Technical tools.} Let $(x_1,\ldots,x_t)\in\X^t$ be a sequence of arms and $(r_2,\ldots,r_{t+1})$ be the corresponding observed (random) rewards, then the unknown parameter $\theta^\star$ can be estimated by GLM estimator. Following~\cite{filippi2010parametric} one gets, for any regularization parameter $\lambda\in\Re^+$, 
\begin{equation}\label{eq:estimate.glm}
\wh{\theta}_{t}^{\text{GLM}} = \arg \min_{\theta \in \mathbb{R}^d} \| \sum_{s=1}^{t-1} \big( r_{s+1} - \mu(x_s^\transp \theta) \big) x_s \|^2_{V_t^{-1}},
\end{equation}
where $V_t$ is the same design matrix as in the linear case. Similar to Prop.~\ref{p:concentration}, we have a concentration inequality for the GLM estimate.

\begin{proposition}[Prop.~1 in appendix.A in~\cite{filippi2010parametric}]\label{p:concentration.glm}
For any $\delta \in (0,1)$, under assumptions~\ref{asm:arm.set},~\ref{asm:param.set},~\ref{asm:subgaussian} and~\ref{asm:link.function}, for any $\F^x_t$-adapted sequence $(x_1,\ldots,x_t,\ldots)$, the prediction returned by the GLM estimator $\wh\theta^{\text{GLM}}_t$ (Eq.~\ref{eq:estimate.glm}) is such that for any fixed $t\geq 1$, 
\begin{equation}  \label{eq:self_normalized1.glm}
\| \wh{\theta}^{\text{GLM}}_{t} - \theta^\star \|_{V_t} \leq \frac{\beta_t(\delta)}{c_\mu},
\end{equation}
and
\begin{equation}\label{eq:self_normalized2.glm}
\begin{split}
\forall x \in \mathbb{R}^d, \hspace{3mm}
&\| \mu(x^\transp \wh{\theta}^{\text{GLM}}_t )- \mu( x^\transp \theta^\star) \| \leq \frac{k_\mu  \beta_t(\delta)}{c_\mu} \| x \|_{V_t^{-1}}, \\
& \| x^\transp \wh{\theta}^{\text{GLM}}_t -  x^\transp \theta^\star \| \leq \frac{\beta_t(\delta)}{c_\mu}  \| x \|_{V_t^{-1}},
\end{split}
\end{equation}
with probability $1-\delta$ (w.r.t.\ the noise sequence $\{\xi_t\}_t$ and any other source of randomization in the definition of the sequence of arms), where $\beta_t(\delta)$ is defined as in Eq.~\ref{eq:beta}.
\end{proposition}

The Asm.~\ref{asm:link.function} on the link function together with the properties of the GLM estimator implies the following:
\begin{enumerate}
\item since the first derivative is strictly positive, $\mu$ is strictly increasing and $x^\star(\theta) = \arg \max_{x \in \mathcal{X}} x^\transp \theta$ so we retrieve the optimal arm of the linear case (and the support function),
\item the concentration inequality of the GLM estimate involves the same ellipsoid as for the RLS (multiplied by a factor $\frac{1}{c_\mu}$). 
\end{enumerate}
These two facts suggest to use then exactly the same TS algorithm as for the linear case (with a $\beta$ multiplied by a factor $\frac{1}{c_\mu}$).\\
\textbf{Sketch of the proof.} From the previous comments, making use of the property of $\mu$, one just need to reduce the GLM case to the standard linear case.
\begin{equation*}
\begin{aligned}
R^{\text{GLM}}(T) &= \sum_{t=1}^T \big( \mu( x^{\star} \theta^\star) - \mu(x_t^\transp \theta^\star)\big), \\
&= \sum_{t=1}^T \big( \mu( x^{\star} \theta^\star) - \mu(x_t^\transp \tilde{\theta}_t)\big) + \sum_{t=1}^T \big( \mu( x_t^\transp \tilde{\theta}_t) - \mu(x_t^\transp \theta^\star)\big) \\
&\leq \sum_{t=1}^T \big( \mu( x^{\star} \theta^\star) - \mu(x_t^\transp \tilde{\theta}_t)\big) +  \sum_{t=1}^T k_\mu \|x\|_{V_t^{-1}} \| \tilde{\theta}_t - \theta^\star\|_{V_t}. \\
\end{aligned}
\end{equation*}
The second term is bounded exactly as $R^{\rls}(T)$. To bound the first one, we make use of the fact that 
\begin{equation*}
\begin{split}
\mu( x^{\star} \theta^\star) - \mu(x_t^\transp \tilde{\theta}_t) &\leq k_\mu \big( J(\theta^\star) - J(\tilde{\theta}_{t}) \big), \quad \text{if} J(\theta^\star) - J(\tilde{\theta}_{t}) \geq 0, \\
\mu( x^{\star} \theta^\star) - \mu(x_t^\transp \tilde{\theta}_t) &\leq c_\mu \big( J(\theta^\star) - J(\tilde{\theta}_{t}) \big), \quad \text{otherwise.} \\
\end{split}
\end{equation*}
Following the proof of the linear case, with high probability, for all $t \geq 1$,
\begin{equation*}
J(\theta^\star) - J(\tilde{\theta}_{t}) \leq \frac{2 \gamma_t(\delta^\prime)}{c_\mu p} \mathbb{E}\big( \|x_t \|_{V_t^{-1}} | \mathcal{F}_{t} \big).
\end{equation*}
Since the r.h.s is strictly positive one can bound the first part of the regret, independently of the sign by,
\begin{equation*}
 \sum_{t=1}^T \big( \mu( x^{\star} \theta^\star) - \mu(x_t^\transp \tilde{\theta}_t)\big) \leq \frac{2 k_\mu \gamma_T(\delta^\prime)}{c_\mu p}  \sum_{t=1}^T \mathbb{E}\big( \|x_t \|_{V_t^{-1}} | \mathcal{F}_{t} \big).
\end{equation*}
Finally, the same proof as in the linear case leads to the following bound for the Generalized Linear Bandit regret.
\begin{lemma}\label{p:generalized_regret}
Under assumptions~\ref{asm:arm.set},\ref{asm:param.set},\ref{asm:subgaussian} and~\ref{asm:link.function}, the cumulative regret of \ts over $T$ steps is bounded as
\begin{equation*}
R^{\text{GLM}}(T) \leq \frac{k_\mu}{c_\mu} \big(\beta_T(\delta^\prime) +  \gamma_T(\delta^\prime)(1 + 2/p) \big) \sqrt{ 2 T d \log \big( 1 + \frac{T}{\lambda} \big) } + \frac{2 k_\mu\gamma_T(\delta^\prime)}{p c_\mu} \sqrt{\frac{8 T } {\lambda} \log \frac{4}{\delta} }
\end{equation*}
with probability $1-\delta$ where $\delta^\prime = \frac{\delta}{4T}$.

\end{lemma}


\section{Regularized Linear Optimization}\label{sec:app_regularized_linear_optimization}

We consider here the Regularized Linear Optimization (RLO) problem as an extension of the Linear Bandit problem. Given a set of arms $\mathcal{X} \subset \mathbb{R}^d$ and an unknown parameter $\theta^\star \in \mathbb{R}^d$, a learner aims at each time step $t = 1,\dots,T$ to select action $x_t \in \mathcal{X} $ which maximizes its associated reward $x_t^\transp \theta^\star + \mu c(x_t)$ where $\mu$ is a known constant and $c$ an arbitrary (yet known) real-valued function. Whenever arm $x$ is pulled, the learner receives a noisy observation $y = x^\transp \theta^\star + \xi$.
 As for LB, we introduce the function $f(x;\theta) = x^\transp \theta + \mu c(x)$, and denote as $x^\star(\theta) = \arg\max_{x \in \mathcal{X}} f(x;\theta)$ and $J(\theta) = \max_{x \in \mathcal{X}} f(x;\theta)$ the optimal action and optimal reward associated with $\theta$. The regret is therefore defined as $R^{RLO}(T) = \sum_{t=1}^Tf(x^\star(\theta^\star);\theta^\star) -  f(x_t;\theta^\star)$.\\
 Since this problem is just the regularized extension of the Linear Bandit, the TS algorithm is similar to Alg.~\ref{alg:ts} where $r_t$ is replaced $y_t$ and $x_t = \arg\max_{x \in \mathcal{X}} f(x,\wt\theta_t)$. Under the same assumptions, the regret shares the same bound and our line of proof holds. First, we decompose the regret 
\begin{align*}
R(T) &= \sum_{t=1}^T \big[ (f(x^\star(\theta^\star);\theta^\star) - f(x_t;\wt\theta_t) ) + ( f(x_t;\wt\theta_t) -   f(x_t;\theta^\star)) \big] \\
&= \underbrace{\sum_{t=1}^T \big[ J(\theta^\star) - J(\wt\theta_t)\big]}_{= R^{\ts}(T)} + \underbrace{\sum_{t=1}^T \big[ x_t^\transp \wt\theta_t  - x_t^\transp \theta^\star\big]}_{ = R^{\rls}(T)}.
\end{align*}
Since Prop.~\ref{p:concentration} holds thanks to the linear observations $y_t$, $R^\rls(T)$ is bounded as in the LB. Finally, to bound $R^\ts(T)$, one just need to ensure that Prop.~\ref{pr:J_support_function}, Lem.~\ref{lem:gradient.optimal.arm} and Lem.~\ref{le:probability_optimistic} hold.
\\
The convexity of the function $f$ with respect to $\theta$ implies the convexity of $J$: $\forall x \in \mathcal{X}$, $\forall \theta, \theta^\prime \in \mathbb{R}^d$, $\forall \alpha \in (0,1)$,
\begin{align*}
J(\alpha \theta + (1 - \alpha) \theta^\prime)  &= \max_{x \in \mathcal{X}} f(x; \alpha \theta + (1 - \alpha) \theta^\prime)\\
& \leq  \max_{x \in \mathcal{X}} \big( \alpha f(x; \theta)  + (1 - \alpha) f(x;\theta^\prime) \big) \leq \alpha J(\theta) + (1-\alpha) J(\theta^\prime).
\end{align*}
Then, $J$ is real-valued and convex which implies its continuous differentiability thanks to Alexandrov's theorem. As a consequence, the first step of the proof holds.\\
The equality between the gradient $\nabla J(\theta)$ and the optimal arm $x^\star(\theta)$ can be derived as in Prop.~\ref{pr:gradient_support_function}: for any $\theta, \bar{\theta} \in \mathbb{R}^d$, by definition, $J(\theta) = f(x^\star(\theta); \theta)$ and $J(\bar{\theta}) \geq f(x^\star(\theta);\bar{\theta})$. Then,
\begin{equation*}
\begin{split}
&J(\bar{\theta}) - f(x^\star(\theta), \bar{\theta}) \geq 0 :=  J(\theta) - f(x^\star(\theta), \theta), \\
&J(\bar{\theta}) \geq J(\theta) + f(x^\star(\theta), \bar{\theta}) - f(x^\star(\theta),\theta) = J(\theta) + x^\star(\theta)^\transp \left( \bar{\theta} - \theta \right), \hspace{3mm} \forall \bar{\theta} \in \mathbb{R}^d,
\end{split}
\end{equation*}
which is the definition of the sub-gradient. Finally, the almost everywhere differentiability of $J$ ensures the sub-gradient to be a singleton and hence equals the gradient. Therefore, Lem.~\ref{lem:gradient.optimal.arm} holds and so is step 2.\\
Finally, since the optimism just relies on the convexity of $J$ and on the over-sampling, it is satisfied in the RLO and step 3 holds. As a result, we obtain the same regret bound as in the LB.

\end{document}